\documentclass{article}

\usepackage[final]{neurips_2025}

\usepackage[utf8]{inputenc} 
\usepackage[T1]{fontenc}    
\usepackage{hyperref}       
\usepackage{url}            
\usepackage{booktabs}       
\usepackage{amsfonts}       
\usepackage{nicefrac}       
\usepackage{microtype}      
\usepackage{xcolor}         
\usepackage{amsmath}
\usepackage{amsthm}
\newtheorem{theorem}{Theorem}
\newtheorem{definition}{Definition}
\newtheorem{assumption}{Assumption}

\usepackage{listings}
\usepackage{tcolorbox}
\usepackage{enumitem}
\usepackage{multirow}
\usepackage{centernot}
\usepackage{caption}
\usepackage{cite}
\usepackage{makecell}

\usepackage{hyperref}
\usepackage{etoc}
\etocdepthtag.toc{mtchapter}
\etocsettagdepth{mtchapter}{subsubsection}
\etocsettagdepth{mtappendix}{none}

\title{Learning with Dual-level Noisy Correspondence for Multi-modal Entity Alignment}

\author{%
  Haobin Li$^{1}$, Yijie Lin$^{1}$, Peng Hu$^{1}$, Mouxing Yang$^{1}$\thanks{Corresponding Authors.}, Xi Peng$^{1,2*}$ \\
  $^{1}$College of Computer Science, Sichuan University, Chengdu, China\\
  $^{2}$National Key Laboratory of Fundamental Algorithms and Models\\ for Engineering Numerical Simulation, Sichuan University, Chengdu, China\\
  \texttt{\{haobinli.gm, linyijie.gm, penghu.ml, yangmouxing, pengx.gm\}@gmail.com}
}

\begin{document}

\maketitle

\vspace{-3mm}

\begin{abstract}
Multi-modal entity alignment (MMEA) aims to identify equivalent entities across heterogeneous multi-modal knowledge graphs (MMKGs), where each entity is described by attributes from various modalities. 
Existing methods typically assume that both intra-entity and inter-graph correspondences are faultless, which is often violated in real-world MMKGs due to the reliance on expert annotations.
In this paper, we reveal and study a highly practical yet under-explored problem in MMEA, termed Dual-level Noisy Correspondence (DNC).
DNC refers to misalignments in both intra-entity (entity-attribute) and inter-graph (entity-entity and attribute-attribute) correspondences.
To address the DNC problem, we propose a robust MMEA framework termed RULE. RULE first estimates the reliability of both intra-entity and inter-graph correspondences via a dedicated two-fold principle. Leveraging the estimated reliabilities, RULE mitigates the negative impact of intra-entity noise during attribute fusion and prevents overfitting to noisy inter-graph correspondences during inter-graph discrepancy elimination. Beyond the training-time designs, RULE further incorporates a correspondence reasoning module that uncovers the underlying attribute-attribute connection across graphs, guaranteeing more accurate equivalent entity identification.
Extensive experiments on five benchmarks verify the effectiveness of our method against the DNC compared with seven state-of-the-art methods. 
The code is available at \href{https://github.com/XLearning-SCU/RULE}{XLearning-SCU/RULE}
\end{abstract}

\vspace{-3mm}

\section{Introduction}
Multi-Modal Entity Alignment~\citep{EVA,ACK-MMEA} (MMEA) aims to identify equivalent entities across different Multi-modal Knowledge Graphs (MMKGs)~\citep{mmkg1,mmkg2}, where each entity is associated with attributes of various modalities (\textit{e.g.}, structural triples and images).
Due to the heterogeneity of attributes from different modalities and graphs from different sources (\textit{e.g.}, Wikidata~\citep{wiki} and YAGO~\citep{yago}), the key challenge of MMEA is to learn a comprehensive representation for each entity with its respective attributes while eliminating the cross-graph discrepancy.
To this end, existing methods usually conduct multi-modal fusion for attributes within the same entity based on the intra-entity correspondences (\textit{i.e.}, entity-attribute pairs), while performing cross-graph alignment by resorting to the inter-graph correspondences (\textit{i.e.}, entity-entity pairs and attribute-attribute pairs).

Despite significant efforts in intra-entity attribute fusion~\citep{MEAformer,PMF} and inter-graph discrepancy elimination~\citep{XGEA,HMEA}, existing MMEA methods heavily rely on the assumption of faultless intra-entity and inter-graph correspondences.
However, as shown in Fig.~\ref{fig: fig1}(a), the assumption is daunting and even impossible to satisfy, leading to the Noisy Correspondence (NC) problem at dual levels.
On the one hand, as the MMKG construction requires expert knowledge, it is inevitable to wrongly associate some entities with irrelevant attributes, resulting in intra-entity NC.
For instance, image of ``Elvis Tsui'' is incorrectly associated with entity ``Jason Momoa'' because of the visual resemblance. 
On the other hand, due to the inherent complexities in attribute and entity association, accurately associating all the inter-graph entities and their corresponding attributes is impractical, leading to inter-graph NC.
For example, movie entity ``Mr. \& Mrs. Smith'' is mistakenly labeled with real-life couple ``Will Smith and Mrs. Smith''.
According to the statistics in Appendix~\ref{sec: Noise Statistics in Real-World Benchmarks}, real-world benchmarks always contain numerous NC (\textit{e.g.}, over 50\% in ICEWS benchmarks).
As shown in Fig.~\ref{fig: fig1}(b), NC would not only undermine the fusion of within-entity attributes but also misleading the inter-graph alignment, both of which significantly degrade the performance.

\begin{figure}[t]
    \centering
    \includegraphics[width=0.9\linewidth]{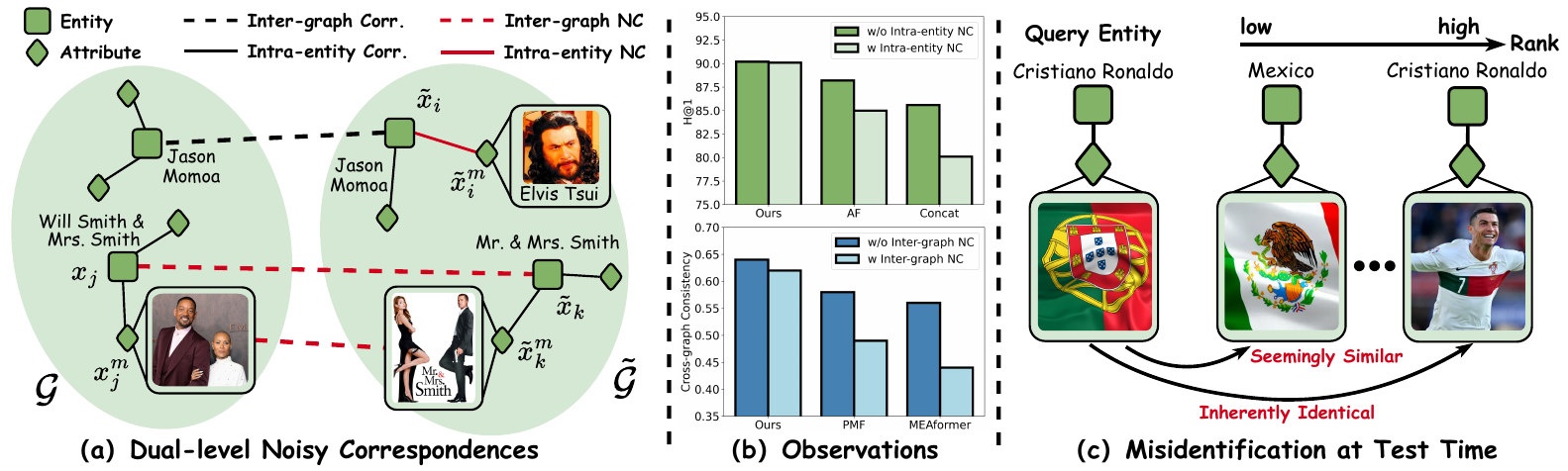}
    \caption{(a) \textbf{Dual-level Noisy Correspondence} occurs in both intra-entity level (\textit{i.e.}, entity-attribute pairs such as ($\tilde{x}_i,\tilde{x}_i^m$)) and the inter-graph level (\textit{i.e.}, entity-entity ($x_{j},\tilde{x}_{k}$) or attribute-attribute pairs ($x_{j}^{m},\tilde{x}_{k}^{m}$)).
    (b) \textbf{Observations}: 
    On the one hand, both vanilla adaptive fusion (AF) and concatenation (Concat) tend to integrate erroneous attributes and thus degrade performance, while our method achieves reliable fusion against inter-graph NC.
    On the other hand, existing methods suffer in cross-graph alignment when encountering inter-graph NC, whereas our method boosts performance by mitigating the negative impact of intra-entity NC.
    (c) \textbf{Misidentification at Test Time}: Seemingly similar attribute pairs may prevent the query entity from being associated with its equivalent entity.
    For instance, the implicit connection between the football player ``Cristiano Ronaldo'' and his home country are often overlooked, resulting in misidentification during entity alignment. }
    \label{fig: fig1}
\end{figure}

Based on the above observations, we reveal a new problem for MMEA, termed Dual-level Noisy Correspondence (DNC).
To conquer the DNC problem, we propose a novel method, dubbed dually RobUst LEarning (RULE), for achieving robust MMEA against DNC.
Specifically, RULE first estimates the reliability of both the intra-entity and inter-graph correspondences by resorting to a dedicatedly-designed two-fold principle and then divides the entity-attribute, entity-entity, and attribute-attribute pairs into different groups.
Based on the estimated reliabilities and division results, RULE alleviates the negative impact of intra-entity NC during intra-entity attribute fusion, while preventing the model from overfitting the inter-graph NC during inter-graph discrepancy elimination.
Beyond the training-time designs, RULE further incorporates a novel correspondence reasoning module to enhance the test-time robustness. 
In brief, this module performs deep reasoning to uncover the underlying attribute-attribute connections across graphs, thus preventing seemingly dissimilar but inherently identical attributes from being neglected (as shown in Fig.~\ref{fig: fig1}(c)) and guaranteeing more accurate equivalent entity identification during inference.

In summary, the major contributions and novelties of this work are given as follows.
\begin{itemize}
[leftmargin=10pt,topsep=0pt,itemsep=0pt]
    \item  We reveal and study a novel and practical problem in MMEA, termed Dual-level Noisy Correspondence (DNC). In brief, DNC refers to the noisy correspondence rooted in the intra-entity (entity-attribute) pairs and inter-graph (entity-entity, attribute-attribute) pairs. We empirically demonstrate that DNC not only undermines multi-modal attribute fusion but also misleads the inter-graph alignment, leading to significant performance degradation for existing MMEA methods.
    \item To achieve robust MMEA against the DNC problem, we propose a novel method termed RULE, which estimates the reliability of
    both the intra-entity and inter-graph correspondences with a dedicatedly-designed two-fold principle and accordingly mitigates the negative impact of DNC during the multi-modal attribute fusion and inter-graph alignment processes. 
    \item During inference, RULE employs a novel correspondence reasoning module to uncover inherently-identical attributes and accordingly achieve more precise cross-graph equivalent entity identification. To the best of our knowledge, this could be one of the first methods to enhance test-time robustness for the MMEA task.
\end{itemize}

\vspace{-3mm}

\section{Method}
In this section, we introduce the proposed RULE for tackling the DNC problem.
In Section~\ref{sec: Problem Formulation}, we present the formal definition of the MMEA task and the DNC problem.
In Section~\ref{sec: Reliability Estimation and Pair Division}, we elaborate on the two-fold principle for the reliability estimation and pair division.  
In Section~\ref{sec: Inter-graph Discrepancy Elimination}-\ref{sec: Intra-entity Attribute Fusion}, we introduce the robust attribute fusion and robust discrepancy elimination modules.
In Section~\ref{sec: test-time reasoning}, we design a test-time correspondence reasoning module to uncover underlying connections between inter-graph attributes, facilitating the equivalent entity identification.

\subsection{Problem Formulation}
\label{sec: Problem Formulation}
Given two heterogeneous multi-modal knowledge graphs (MMKGs), denoted as $\mathcal{G} = \{x_i, \{x_i^m\}_{m=1}^{M}\}_{i=1}^{N}$ and $\tilde{\mathcal{G}} = \{\tilde{x}_j, \{\tilde{x}^m_j\}_{m=1}^{\tilde{M}}\}_{j=1}^{\tilde{N}}$, where $x_i$ and $\tilde{x}_j$ are entities in $\mathcal{G}$ and $\mathcal{\tilde{G}}$, respectively. 
Each entity $x_{i}\in \mathcal{G}$ is associated with $M$ attribute-specific attributes $\{x_i^m\}_{m=1}^{M}$, such as structured triples, textual descriptions, and images.

Within a single graph, the association between an entity and its attributes is captured by entity-attribute pairs $(x_i, x_i^m, h_i^m)$, where $h_i^m \in \{0, 1\}$ is a binary indicator, $h_i^m = 1$ indicates the valid \textit{intra-entity correspondence}, and $h_i^m = 0$ denotes no correspondence between $x_i$ and $x_i^m$.
Across graphs, \textit{inter-graph correspondences} govern the alignment of both the entity-entity pairs and attribute-attribute pairs.
To be specific, the entity-entity pair is represented by $(x_i, \tilde{x}_j, y_{ij})$, where the correspondence $y_{ij} = 1$ if $x_i$ and $\tilde{x}_j$ refer to the same real-world concepts, and $y_{ij} = 0$ otherwise. 
Similarly, the attribute-attribute pair is denoted by $(x_i^m, \tilde{x}_j^m, y_{ij}^m)$, where the correspondence $y_{ij}^m = 1$ \textit{i.f.f.} both attributes are linked to correct entities (\textit{i.e.}, $h_i^m = 1~\&~\tilde{h}_j^m = 1$) and the corresponding entities $x_i$ and $\tilde{x}_j$ are aligned (\textit{i.e.}, $y_{ij} = 1$). 
In other words, once the inter-graph entities are associated, their corresponding attributes could be treated as matched.

Given a query entity $x_i \in \mathcal{G}$, the goal of multi-modal entity alignment is to identify its equivalent entity $\tilde{x}_j$ from the other $\tilde{\mathcal{G}}$ such that $y_{ij} = 1$. 
To this end, existing approaches typically follow a two-stage pipeline:
i) intra-entity attribute fusion: for each entity $x_i$, attribute representations are first extracted using attribute-specific encoders $z_i^m = f^m(x_i^m)$, and then aggregated to form a unified entity representation $z_i$;
ii) inter-graph discrepancy elimination: based on the fused entity representations $z_i$ and $\tilde{z}_j$, contrastive learning~\citep{simclr} is employed to mitigate the inter-graph discrepancy. 
However, in practice, this pipeline assumes that both the intra-entity correspondences (\textit{i.e.}, entity-attribute $h_{i}^{m}$) and inter-graph correspondences (\textit{i.e.}, entity-entity $y_{ij}$ and attribute-attribute $y_{ij}^{m}$) are perfectly labeled. 
However,  due to annotation errors, such an assumption is often violated, leading to the DNC challenge.
As discussed in Introduction, the DNC problem would undermine the inter-graph and intra-entity learning, leading to remarkable performance degradation.

\begin{figure}[t]
    \centering
    \includegraphics[width=0.9\linewidth]{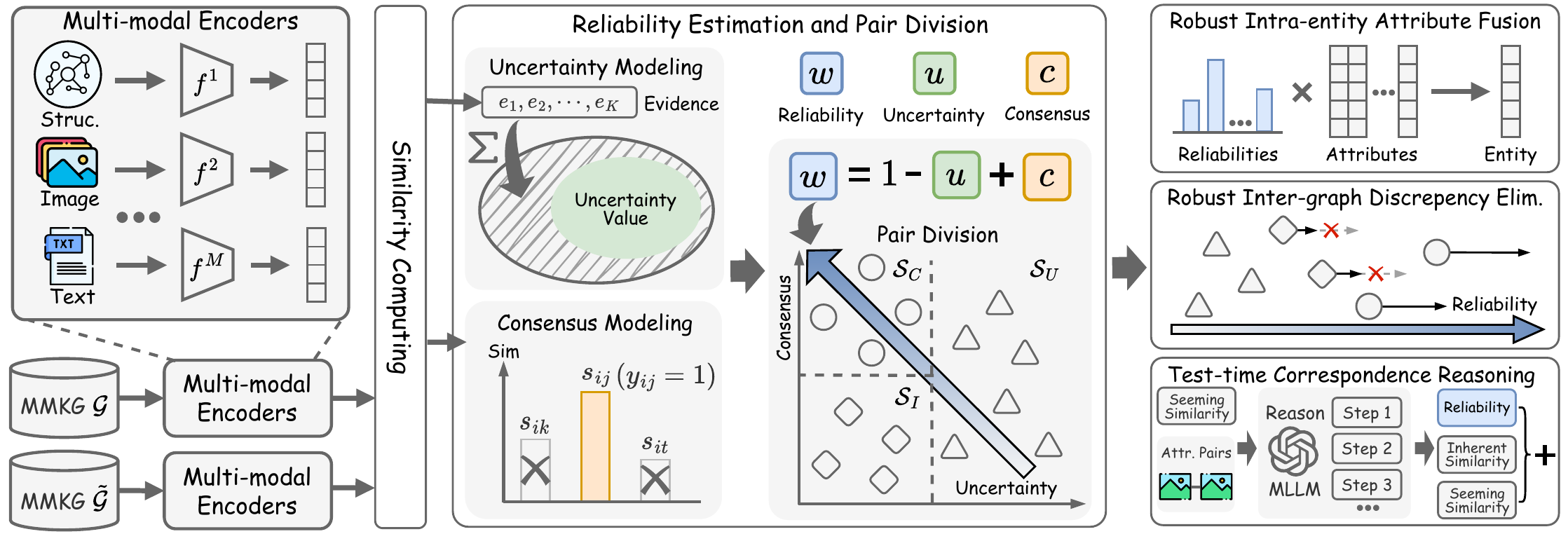}
    \caption{Overview of our method RULE. Given two MMKGs $\mathcal{G}$ and $\tilde{\mathcal{G}}$, RULE first projects the entity attributes into a shared latent space and computes cross-graph attribute similarities. 
    These similarities are used to estimate the reliability of inter-graph correspondences and categorize cross-graph pairs into three subsets: $\mathcal{S}_C$, $\mathcal{S}_I$, and $\mathcal{S}_U$.
    Subsequently, the robust intra-entity attribute fusion module and robust inter-graph discrepancy elimination are employed to mitigate the impact of both intra-entity and inter-graph noisy correspondences.
    Beyond training-time robustness, a test-time correspondence reasoning module uncovers latent attribute-attribute connections across graphs, enabling more accurate equivalent entity identification during inference.}
    \label{fig: overview}
\end{figure}

\subsection{Reliability Estimation and Pair Division}
\label{sec: Reliability Estimation and Pair Division}
To facilitate robust inter-graph discrepancy elimination and intra-entity attribute fusion, we first estimate the reliability of both the intra-entity and inter-graph correspondences by resorting to a two-fold principle, \textit{i.e.}, uncertainty and consensus.
Without loss of generality, in the following, we take the inter-graph entity-entity correspondence as a showcase to elaborate on the process of correspondence reliability estimation. 
For a given entity $x_i$, the reliability $w_i$ between $x_i$ and its associated counterpart $\tilde{x}_j~(y_{ij}=1)$ is estimated using the following principle:
\begin{equation}
    w_i=\left(1-u_i\right) \alpha+c_i (1-\alpha),
    \label{eq: reliablity_weights}
\end{equation}
where $\alpha$ is the balanced hyper-parameters (fixed as $0.5$ for simplicity, see Appendix~\ref {appendix: ablation of reliability} for more choices), $u_{i}$ and $c_{i}$ denote the uncertainty and consensus for the correspondence and will be detailed in the following sections.

\subsubsection{Uncertainty Modeling}
For a given entity, uncertainty in this work refers to whether its correspondence is trustworthy or not, which could serve as the principle to identify NC.
According to the Dempster-Shafer Theory~\citep{Dempster–Shafer}, uncertainty could be quantified by evidence, which measures how the data support the association between a query and a candidate.
Specifically, the more evidences the entity accumulates, the lower uncertainty it embraces.
Formally, evidence of the entity pairs $(x_i, \tilde{x}_j)$ is defined as
\begin{equation}
    e_{ij} = \exp\left(\tanh\left(s_{ij}/ \tau\right)\right),
    \label{eq: evi}
\end{equation}
where $s_{ij}=z_i\cdot \tilde{z}_j$ denotes the dot product between the entity representation $z_i$ and $\tilde{z}_j$, $\tau$ is the temperature, and the evidence vector for $x_{i}$ is $\boldsymbol{e}_{i}=[e_{i1};e_{i2};\cdots;e_{i\tilde{N}}]$. 
Following Subjective Logic~\citep{Evidentiallearning}, we associates the evidence vector $\boldsymbol{e}_{i}$ with the parameters of the
Dirichlet distribution $\boldsymbol{\alpha}_{i}=[\alpha_{i1},\alpha_{i2},\cdots,\alpha_{i\tilde{N}}]$, where $\alpha_{ij}=e_{ij}+1$. 
\begin{definition}
    \textbf{Uncertainty.}
    For a given entity $x_{i}$, the uncertainty and the corresponding belief mass are defined as
    \begin{equation}
         u_{i}=\frac{\tilde{N}}{Q_{i}} \ \text{and} \  b_{ij}=\frac{e_{ij}}{Q_{i}}=\frac{\alpha_{ij}-1}{Q_{i}},
        \label{eq: uncertainty}
    \end{equation}
    where $Q_{i}=\sum_{j}^{\tilde{N}}\left(e_{ij}+1\right)=\sum_{j}^{\tilde{N}}\alpha_{ij}$ and $u_{i}+\sum_{j}^{\tilde{N}}b_{ij}=1$.
\end{definition}
The $Q_{i}$ denotes the Dirichlet distribution strength, and the belief mass assignment $\boldsymbol{b}_{i}=[b_{i1};b_{i2};\cdots;b_{i\tilde{N}}]$, \textit{i.e.}, subjective opinion, corresponds to the Dirichlet distribution with parameters $\boldsymbol{\alpha}_{i}$.
Such a formulation encourages the mismatched entity-entity pairs to yield limited evidence, as the given entity fails to associate with any entity in the other MMKG, resulting in high uncertainty.

\subsubsection{Consensus Modeling}
\label{sec: consensus modeling}
Although the formulated uncertainty would help to identify noisy correspondence, we observe that a low uncertainty does not necessarily indicate a correct correspondence. 
Formally,
\begin{theorem}
A low uncertainty $u_i$ does not necessarily imply that the highest belief is assigned to the annotated correspondence $\boldsymbol{y}_{i}$, \textit{i.e.},
\begin{equation}
    z_i \text{ with low } u_i \centernot\Rightarrow \arg\max\, \boldsymbol{b}_{i} = \arg\max \boldsymbol{y}_{i}.
\end{equation}
\label{theorem: wrong uncertainty}
\end{theorem}
\vspace{-6mm}
The Proof is placed in Appendix~\ref{appendix: theorem}.
Here, $\boldsymbol{y}_{i}=[y_{i1};y_{i2};\cdots;y_{i\tilde{N}}]$ is a one-hot vector indicating the inter-graph entity-entity correspondence of entity $x_i$.
Such a theorem highlights that uncertainty is insufficient to determine whether the belief is concentrated on the annotated correspondence. 
Therefore, we propose the consensus principle as follows.
\begin{definition}
    \textbf{Consensus.}
    For a given entity $x_{i}$, the consensus is defined as
    \begin{equation}
        c_{i}=\max (0, \boldsymbol{s}_{i}\cdot \boldsymbol{y}_{i}),
        \label{eq: cons}
    \end{equation}
    where $\boldsymbol{s}_{i}=[s_{i1},s_{i2},\cdots,s_{iN}]$ denotes the similarity vector, $\max (0,\cdot)$ ensures the consensus is non-negative.
\end{definition}
Intuitively, a low consensus $c_{i}$ indicates that the given correspondence is unreliable, thus serving as another principle to identify noisy correspondence.
However, during inference, the annotated correspondence $\boldsymbol{y}_{i}$ in Eq.~\ref{eq: cons} is unavailable.
To remedy this, we propose to estimate the correct correspondence through a greedy strategy based on marginal contribution.
Here, we begin with a definition of marginal contribution.
\begin{definition}
    For a given entity $x_{i}$, the marginal contribution of its $m$-th attribute is defined as
    \begin{equation}
        \Delta = v(\pi \cup \{m\}) - v(\pi),
        \label{eq: marginal contribution}
    \end{equation}
    where $v(\cdot)$ indicates the value function, $\pi \subseteq \Pi \setminus \{m\}$ denotes a subset $\pi$ of attributes excluding the $m$-th one, $\Pi$ is the complete set of available attributes.
\end{definition}
In the implementation, we define the value function as $v(\pi) = \max \left( \frac{1}{|\pi|} \sum_{j \in \pi} \boldsymbol{s}_{i}^{j} \right)$ and $v(\pi \cup \{m\}) = \max \left( \frac{1}{|\pi| + 1} \sum_{j \in \pi \cup \{m\}} \boldsymbol{s}_{i}^{j} \right)$, where $|\cdot|$ denotes the number of attributes.
Inspired by Shannon’s principle that “the essence of information is to eliminate uncertainty”,  we expect that the informative attributes would contribute to establishing reliable correspondence for the entity-attribute pairs.
Thus,
\begin{assumption}
    For a given entity $x_{i}$, if $x_{i}^{m}$ is correctly associated with $x_{i}$, then $\Delta\geq0$. 
    Conversely, if $x_{i}^{m}$ is irrelevant to $x_{i}$, then $\Delta<0$.
    \label{assumptaion: EA_NC}
\end{assumption}
Assumption~\ref{assumptaion: EA_NC} provides a feasible way to estimate the correct correspondence.
Specifically, incorporating attributes until the marginal contribution no longer improves, and the established subset $\pi$ would help to indicate a reliable correspondence.
To implement this, we adopt the following greedy strategy,
\begin{equation}
    \pi^* = \pi_0 \cup \left\{ m \in (\Pi \setminus \pi_0) \;\middle|\; v(\pi_0 \cup \{m\}) - v(\pi_0) > 0 \right\},
    \label{eq: greedy}
\end{equation}
where $\pi_0$ denotes the initial subset with $|\pi_0| = \left\lfloor \frac{M}{2}+1 \right\rfloor$ when $M \geq 3$. See more details in Appendix~\ref{appendix: DRF}.
With the selected subset $\pi^{*}$, the estimated correspondence is finally given as $\boldsymbol{y}_{i} = \operatorname{one\text{-}hot} ( \arg\max ( \frac{1}{|\pi^*|} \sum_{m \in \pi^*} \boldsymbol{s}_i^m ) )$, where $\operatorname{one-hot}$ denotes the vector conversion.

\subsubsection{Pair Division}
With the formulated uncertainty and consensus, we could further identify the inter-graph NC.
Specifically, we propose to divide the inter-graph pairs with $y_{ij}=1$ into three portions: noisy portion with high uncertainty $\mathcal{S}_{U}=\{x_{i}, \tilde{x}_{j}\mid u_{i}> \beta_u\}$, noisy portion with low consensus $\mathcal{S}_{I}=\{x_{i}, \tilde{x}_{j}\mid u_i \leq \beta_u \text{ and } c_i < \beta_c\}$ and clean portion $\mathcal{S}_{C}=\{x_{i}, \tilde{x}_{j}\mid u_i \leq \beta_u \text{ and } c_i \geq \beta_c\}$.
The thresholds $\beta_u$ and $\beta_c$ are determined in a self-adaptive manner via
\begin{equation}
    \beta_u = \min(u^{TP}, 1 - \beta), \quad \beta_c = \max(\beta, c^{TP}),
    \label{eq: threshold}
\end{equation}
where $u^{TP} = \max_{i \in \mathcal{S}^{\text{TP}}} u_i$, $ c^{TP} = \min_{i \in \mathcal{S}^{\text{TP}}} c_{i}$, and $\beta$ indicates the threshold hyperparameter. Here, $\mathcal{S}^{\text{TP}}= \left\{ i \,\middle|\, \arg\max(\boldsymbol{s}_i) = \arg\max(\boldsymbol{y}_i) \right\}$ denotes the set of true positive pairs.
With the above pair division, the inter-graph pairs could be divided into $\mathcal{S}_{U}$, $\mathcal{S}_{I}$, and $\mathcal{S}_{C}$, which are further used for inter-graph discrepancy elimination.

\subsection{Robust Inter-graph Discrepancy Elimination}
\label{sec: Inter-graph Discrepancy Elimination}
With the established reliability and pair division results, we could obtain three subsets: $\mathcal{S}_{U}$, $\mathcal{S}_{I}$, and $\mathcal{S}_{C}$. 
Since the pairs in $\mathcal{S}_{U}$ exhibit high uncertainty, they are considered unreliable and be excluded from the discrepancy elimination.
As discussed in Section~\ref{sec: consensus modeling}, 
inter-graph pairs with low consensus do not necessarily indicate correct matches, thus the pairs in $\mathcal{S}_{I}$ cannot be regarded as reliable.
Accordingly, we propose a novel Dually Robust Learning (DRL) that employs tailored strategies for the three subsets, thereby achieving robustness against inter-graph noisy correspondence.
Formally, the overall objective is defined as
\begin{equation}
    \mathcal{L}=\mathcal{L}_{DR}+\lambda \mathcal{L}_{Reg},
    \label{eq: overall}
\end{equation}
where $\mathcal{L}_{DR}$ and $\mathcal{L}_{Reg}$ denotes the dually robust loss and regularization loss, $\lambda$ denotes the trade-off parameter. 
Specifically, the dually robust loss and regularization loss are given by,
\begin{equation}
    \mathcal{L}_{DR}\!=\!\mathcal{L}_{DR}(\boldsymbol{\alpha}_{i},\boldsymbol{\hat{y}}_{i}) \!+\! \sum_{m=1}^{M}\mathcal{L}_{DR}(\boldsymbol{\alpha}_{i}^{m},\boldsymbol{\hat{y}}_{i}^{m}),\quad
    \mathcal{L}_{Reg}\!=\!\mathcal{L}_{Reg}(\boldsymbol{\alpha}_{i},\boldsymbol{\hat{y}}_{i}) \!+\! \sum_{m=1}^{M}\mathcal{L}_{Reg}(\boldsymbol{\alpha}_{i}^{m},\boldsymbol{\hat{y}}_{i}^{m}),
\end{equation}
where $\boldsymbol{\alpha}_{i}^{m}$ and $\boldsymbol{\hat{y}}_{i}^{m}$ are the Dirichlet parameter and refined correspondence for $x_i^{m}$. 
More specifically, for the given entity $x_{i}$, the dually robust loss is defined as
\begin{equation}
\mathcal{L}_{DR}(\boldsymbol{\alpha}_{i},\boldsymbol{\hat{y}}_{i})= \mathbb{I}\left(i\notin \mathcal{S}_{U}\right)\int\left\|\boldsymbol{\hat{y}}_i-\boldsymbol{p}_i\right\|_2^2 \ D(\boldsymbol{p}_{i} \mid \boldsymbol{\alpha}_{i}) \ d \mathbf{p}_i,
\label{eq: dually robust loss}
\end{equation}
where $D(\boldsymbol{p}_{i} \mid \boldsymbol{\alpha}_{i})$ denotes the density function of the Dirichlet distribution over the query probability $\boldsymbol{p}_{i}=[p_{i1},p_{i2},\cdots,p_{i\tilde{N}}]$, and $\mathbb{I}(\cdot)$ indicates an indicator function evaluating to 1 \textit{i.f.f} the condition is satisfied.
The refined correspondence $\boldsymbol{\hat{y}}_{i}$ is defined as follows,
\begin{equation}
    \boldsymbol{\hat{y}}_i =
        \begin{cases}
        \boldsymbol{y}_i, & \text{if } i\in \mathcal{S}_{C} \\
        c_{i} \boldsymbol{y}_i + (1 - c_{i}) \operatorname{Softmax}(\boldsymbol{s}_{i}), & \text{if } i\in \mathcal{S}_{I}
        \end{cases}.
    \label{eq: ref label}
\end{equation}
Such behavior enhances robustness against inter-graph noisy correspondences for the following reasons.
On the one hand, the upper bound of query probability is proportional to $Q_i$ (Theorem~\ref{pro: uncertainty upper bound}), thus preventing over-optimization when the accumulate $Q_i$ is limited. 
On the other hand, excluding high-uncertainty correspondences in $\mathcal{S}_{U}$ and refining the low-consensus correspondences in $\mathcal{S}_{I}$ would prevent erroneous optimization caused by NC.

Although the proposed dually robust loss in Eq.~\ref{eq: dually robust loss} could encourage higher evidence for inter-graph pairs with reliable correspondence, it is unable to guarantee that unassociated inter-graph pairs generate limited evidence. 
To achieve this, a Kullback-Leibler (KL) divergence term is adopted to penalize the evidence of the unassociated inter-graph pairs, \textit{i.e.},
\begin{equation}
    \mathcal{L}_{\mathrm{Reg}}(\boldsymbol{\alpha}_{i},\boldsymbol{\hat{y}}_{i}) = \mathrm{KL}\left[D\left(\boldsymbol{p}_i \mid \boldsymbol{\tilde{\alpha}}_i\right) \,\|\, D\left(\boldsymbol{p}_i \mid \mathbf{1}\right)\right]
    \label{eq: kl}
\end{equation}
where $\mathbf{1}\in \mathbb{R}^{\tilde{N}}$ is a $\tilde{N}$-dimensional vector of ones, $\boldsymbol{\tilde{\alpha}}_{i}=\boldsymbol{\hat{y}}_{i}+(1-\boldsymbol{\hat{y}}_{i})\odot \boldsymbol{\alpha}_{i}$ denotes the Dirichlet parameters which help to penalize the evidence of unassociated correspondence, $\Gamma(\cdot)$ and $\psi(\cdot)$ are the gamma and digamma function, respectively. 

\subsection{Robust Intra-entity Attribute Fusion}
\label{sec: Intra-entity Attribute Fusion}
As discussed in Section~\ref{sec: Problem Formulation}, inter-graph attribute associations emerge as the by-product of establishing entity-attribute and entity-entity correspondences.
Therefore, for correctly paired entities, the attribute-attribute correspondence is incorrect, \textit{i.f.f}, the corresponding entity-attribute correspondence is wrongly established.
Thus, the inter-graph reliability $w_i^m$ could be employed to identify unreliable intra-entity attributes and weaken the emphasis on them during attribute fusion.
Specifically, for a given entity $x_{i}$, we employ the following Dually Robust Fusion (DRF) module to obtain the integrated representation,
\begin{equation}
    z_i = \oplus_{m \in M} \left( w_i^m \cdot z_i^m \right),
\end{equation}
where $\oplus$ indicates the concatenation operator.
Such behavior achieves robustness against noisy entity-attribute pairs by fusing the multi-modal attributes with adaptive weights. In other words, attributes with higher reliability are emphasized, while those with lower reliability are weakened.

\subsection{Test-time Correspondence Reasoning}
\label{sec: test-time reasoning}

\begin{table*}[t]
\centering
\caption{Comparisons with state-of-the-art methods on Non-name benchmarks under DNC setting. ``Inherent DNC'' refers to the setting without any additional injected noise. H@$k$ indicates the top-$k$ retrieval accuracy while MRR denotes the mean reciprocal rank. The best and second best results are marked in \textbf{bold} and \underline{underline}.
}
\label{tab: main table visual}
    \resizebox{0.95\linewidth}{!}
    {  
    \centering
    \begin{tabular}{c|l|ccc|ccc|ccc|ccc|ccc|c}
        \toprule
        \multirow{2}{*}{Setting} & \multirow{2}{*}{Method} & \multicolumn{3}{c|}{ICEWS-WIKI} & \multicolumn{3}{c|}{ICEWS-YAGO} & \multicolumn{3}{c|}{DBP15K ZH-EN} & \multicolumn{3}{c|}{DBP15K JA-EN} & \multicolumn{3}{c|}{DBP15K FR-EN} & \multirow{2}{*}{\makecell{Avg.\\H@1}} \\
         & & H@1 & H@5 & MRR & H@1 & H@5 & MRR & H@1 & H@5 & MRR & H@1 & H@5 & MRR & H@1 & H@5 & MRR & \\
        \midrule
        \multirow{8}{*}{\parbox[c][1cm][c]{1cm}{\centering Inherent \\ DNC}} 
        & EVA           & 29.6 & 40.7 & 35.1 & 8.0  & 13.7 & 11.1 & 70.7 & 86.8 & 77.9 & 73.6 & 89.5 & 80.6 & 74.3 & 90.5 & 81.4 & 51.2 \\
        & MCLEA         & 43.2 & 63.1 & 52.4 & 30.1 & 47.7 & 38.8 & 76.6 & 90.8 & 83.0 & 77.8 & 92.0 & 84.1 & 78.7 & 92.7 & 84.9 & 61.3 \\
        & XGEA          & 49.8 & 61.5 & 55.5 & 35.5 & 46.7 & 41.2 & 81.1 & 93.0 & 86.3 & 82.6 & 94.3 & 87.8 & 83.1 & 94.7 & 88.3 & 66.4 \\
        & MEAformer     & \underline{53.5} & \underline{70.1} & \underline{61.3} & 35.0 & 51.2 & 42.8 & 82.4 & 93.5 & 87.3 & 81.9 & 94.2 & 87.3 & 82.1 & 94.4 & 87.5 & 67.0 \\
        & UMAEA         & 51.2 & 70.0 & 59.9 & 32.4 & 49.4 & 40.6 & 79.1 & 93.2 & 85.3 & 79.6 & 93.9 & 85.8 & 81.2 & 95.0 & 87.3 & 64.7 \\
        & PMF           & 52.6 & 67.9 & 59.9 & \underline{38.3} & \underline{53.2} & \underline{45.4} & \underline{83.9} & \underline{94.6} & \underline{88.9} & \underline{83.9} & \underline{94.9} & \underline{89.0} & \underline{84.4} & \underline{95.3} & \underline{89.6} & \underline{68.6} \\
        & HHEA   & 49.0 & 64.6 & 56.4 & 37.5 & 50.4 & 43.8 & 48.7 & 62.5 & 55.5 & 49.9 & 60.6 & 55.4 & 52.8 & 63.6 & 58.2 & 47.6 \\
        & Ours          & \textbf{64.2} & \textbf{76.7} & \textbf{70.0} & \textbf{48.8} & \textbf{60.5} & \textbf{54.6} & \textbf{85.6} & \textbf{94.8} & \textbf{89.7} & \textbf{85.2} & \textbf{95.4} & \textbf{89.6} & \textbf{85.1} & \textbf{95.4} & \textbf{89.6} & \textbf{73.8} \\
        \midrule
        \multirow{8}{*}{\parbox[c][1cm][c]{1cm}{\centering 20\% \\ DNC}}
        & EVA           & 15.2 & 21.6 & 18.4 & 0.2  & 0.4  & 0.4  & 51.0 & 70.2 & 59.7 & 54.5 & 73.4 & 63.1 & 53.4 & 73.8 & 62.6 & 34.9 \\
        & MCLEA         & 34.5 & 53.6 & 43.5 & 24.6 & 40.4 & 32.5 & 69.9 & 85.7 & 77.0 & 70.1 & 85.6 & 77.2 & 70.7 & 87.3 & 78.1 & 54.0 \\
        & XGEA          & 40.4 & 48.4 & 44.6 & 22.6 & 27.6 & 25.7 & 76.3 & \underline{90.7} & 82.7 & 76.6 & \underline{91.1} & 83.0 & 76.9 & 91.2 & 83.7 & 58.6 \\
        & MEAformer     & \underline{50.8} & \underline{67.5} & \underline{58.4} & 35.9 & \underline{50.7} & 43.0 & \underline{77.7} & 90.6 & \underline{83.4} & \underline{77.8} & 90.9 & \underline{83.6} & \underline{78.0} & \underline{91.5} & \underline{84.0} & \underline{64.0} \\
        & UMAEA         & 48.4 & 64.6 & 56.1 & 31.1 & 46.5 & 38.6 & 74.5 & 89.6 & 81.3 & 73.6 & 89.4 & 80.7 & 74.3 & 89.9 & 81.1 & 60.4 \\
        & PMF           & 45.4 & 60.6 & 52.6 & 36.2 & 49.9 & 42.7 & 76.7 & 90.2 & 82.7 & 76.5 & 89.9 & 82.5 & 77.1 & 90.7 & 83.2 & 62.4 \\
        & HHEA   & 47.8 & 61.8 & 54.4 & \underline{37.4} & 49.5 & \underline{43.3} & 48.7 & 58.8 & 53.8 & 49.0 & 58.7 & 54.0 & 52.5 & 61.7 & 57.1 & 47.1 \\
        & Ours          & \textbf{62.4} & \textbf{75.1} & \textbf{68.5} & \textbf{48.3} & \textbf{59.5} & \textbf{53.9} & \textbf{81.1} & \textbf{92.0} & \textbf{86.0} & \textbf{80.5} & \textbf{92.2} & \textbf{85.6} & \textbf{80.5} & \textbf{92.2} & \textbf{85.8} & \textbf{70.6} \\
        \midrule
        \multirow{8}{*}{\parbox[c][1cm][c]{1cm}{\centering 50\% \\ DNC}}
        & EVA           & 0.5  & 0.8  & 0.9  & 0.0  & 0.1  & 0.2  & 17.2 & 30.5 & 23.6 & 18.3 & 32.0 & 24.8 & 14.0 & 27.2 & 20.3 & 10.0 \\
        & MCLEA         & 24.5 & 39.9 & 31.9 & 17.4 & 31.1 & 24.1 & 55.2 & 72.1 & 63.1 & 54.0 & 70.4 & 61.5 & 54.6 & 70.9 & 62.0 & 41.1 \\
        & XGEA          & 39.5 & 47.0 & 43.4 & 23.7 & 27.8 & 26.3 & 67.9 & 83.6 & 74.9 & \underline{68.0} & \underline{83.8} & \underline{75.0} & \underline{68.0} & \underline{83.9} & \underline{75.1} & 53.4 \\
        & MEAformer     & 42.4 & \underline{58.8} & 50.1 & 30.6 & 45.0 & 37.5 & \underline{68.1} & \underline{83.7} & \underline{75.1} & 62.9 & 80.3 & 70.8 & 65.8 & 82.6 & 73.4 & \underline{54.0} \\
        & UMAEA         & 37.8 & 55.0 & 46.0 & 25.4 & 40.0 & 32.5 & 64.8 & 82.1 & 72.7 & 58.1 & 78.5 & 67.2 & 61.8 & 80.9 & 70.3 & 49.6 \\
        & PMF           & 35.1 & 48.8 & 41.8 & 29.6 & 42.4 & 35.8 & 67.1 & 82.6 & 74.2 & 65.6 & 80.7 & 72.5 & 66.1 & 81.5 & 73.1 & 52.7 \\
        & HHEA   & \underline{43.9} & 57.7 & \underline{50.4} & \underline{34.3} & \underline{46.2} & \underline{40.2} & 45.5 & 55.2 & 50.3 & 46.4 & 55.4 & 51.2 & 50.1 & 59.1 & 54.7 & 44.1 \\
        & Ours          & \textbf{58.2} & \textbf{69.7} & \textbf{63.6} & \textbf{46.9} & \textbf{57.4} & \textbf{52.0} & \textbf{73.4} & \textbf{85.9} & \textbf{79.2} & \textbf{71.8} & \textbf{84.9} & \textbf{77.8} & \textbf{71.4} & \textbf{84.8} & \textbf{77.5} & \textbf{64.3} \\
        \bottomrule
    \end{tabular}
    }
\end{table*}

As discussed in the Introduction, the seemingly similar attributes might hinder the identification of equivalent entities.
To solve the problem, we propose Test-time correspondence Reasoning (TTR) module, which uncovers the underlying attribute-attribute connections across graphs, thus improving the equivalent entity identification during inference. 
Specifically, the refined entity-entity similarity scores are given by,
\begin{equation}
    \boldsymbol{\hat{s}}_{i} = \sum_{m\in M}\hat{w}_{i}^{m}\cdot \boldsymbol{\hat{s}}_{i}^{m},
    \label{eq: ttr fusion}
\end{equation}
where $\boldsymbol{\hat{s}}_{i}^{m}$ represents the similarity scores of the $m$-th attribute output by the MLLM and $\hat{w}_{i}^{m}$ denotes the corresponding reliability weight.
Such behaviour could mitigate the negative impact of intra-entity NC, which might undermine attribute fusion during test time.
More specifically, we employ Chain-of-Thought (CoT) to guide the MLLM toward step-by-step reasoning. Mathematically, 
\begin{equation}
    \boldsymbol{\hat{s}}_{i}^{m} = \operatorname{Softmax} \left( \oplus_{j \in \mathcal{T}_i^m} \left( \operatorname{CoT}\left[ x_i^{m}, \tilde{x}_{j}^{m}, \boldsymbol{s}_i^m \right] \right) \right),
    \label{eq: mllm reasoning}
\end{equation}
where $\mathcal{T}_{i}^{m}$ denotes the set of correspondences with the highest similarity in prior results $\boldsymbol{s}_{i}^{m}$, $\operatorname{CoT}$ indicates the reasoning process.
Although a feasible solution is to prompt the MLLM with simple instructions such as ``Identify the similarities between these attributes.'', such vanilla prompts fail to fully activate the deep reasoning capabilities of MLLM.
In contrast, the proposed CoT-based reasoning would enable the MLLM to leverage prior results and detailed steps for reasoning, preventing deviations from the prior knowledge while facilitating the mining of underlying connections.
See Appendix~\ref{appendix: TTR} and Appendix~\ref{appenidx: ttr_cases} for more details.
Finally, the joint similarity score could be derived as $\boldsymbol{s}_{i}^{joint}=\boldsymbol{s}_{i} + \boldsymbol{\hat{s}}_{i}$ and the identified equivalent entity is given by $\arg \max \boldsymbol{s}_{i}^{joint}$.

\section{Experiments}
In this section, we conduct experiments on five widely-used MMEA datasets to validate the effectiveness of RULE. 
Due to space limitation, we present more experiments in Appendix~\ref{appendix: more experiments}.

\subsection{Implementation Details and Experimental Settings}
Our method contains two networks, the attribute-specific encoders $f^m$ and the test-time correspondence reasoning module. 
Specifically, we first utilize a pre-trained CLIP model~\citep{CLIP} to extract features from visual and textual attributes.
After that, we employ the attribute-specific encoders to obtain the latent embeddings following~\citep{PMF,XGEA}.
For the test-time correspondence module, we use Qwen2.5-VL-72B-Instruct~\citep{QWEN2.5} as default to facilitate the test-time correspondence reasoning module (Section~\ref{sec: test-time reasoning}).
Regardings hyperparameters, we set the trade-off parameter $\lambda$ in Eq.~\ref{eq: overall}, the threshold $\beta$ in Eq.~\ref{eq: threshold} are fixed as $1e^{-4}$, $0.3$ for all the experiments, respectively. The temperature $\tau$ in Eq.~\ref{eq: evi} is set to $0.07$ following~\citep{simclr}.

We evaluate our method on five benchmark datasets: ICEWS-WIKI~\citep{SimpleHHEA}, ICEWS-YAGO, DBP15K$_{\text{ZH-EN}}$~\citep{EVA}, DBP15K$_{\text{JA-EN}}$, and DBP15K$_{\text{FR-EN}}$. Details of the dataset and evaluation metric are provided in Appendix~\ref{appedix: datasets} and~\ref{appendix: metric}.
As discuss in Introduction, the MMEA benchmarks including ICEWS always contaminated by DNC which denoted as ``Inherent DNC" in the paper.
To further evaluate the robustness toward DNC, we manually inject noise to conduct more comprehensive evaluations by following the widely-adopted strategies in the noisy correspondence/label learning community~\citep{noisylabel,NCR}.
Specifically, the artificial noise are injected in the following three aspects:
i) \textit{entity-entity NC}: one entity in an aligned entity pair is randomly replaced with a different entity;
ii) \textit{entity-attribute NC}: a visual or textual attribute is randomly reassigned to a different entity;
iii) \textit{attribute-attribute NC}: 
visual attributes are perturbed with Gaussian noise, while textual attributes are corrupted via random character replacements.
The artificial noise levels are set as $20\%$ and $50\%$ in our experiments, which represents the proportion of corrupted E-E/E-A/A-A pairs.

\subsection{Comparisons with State-Of-The-Arts}
In this section, we compare our method RULE with seven state-of-the-art MMEA methods under the Dual-level Noisy Correspondence setting, including EVA~\citep{EVA}, MCLEA~\citep{MCLEA}, XGEA~\citep{XGEA}, MEAformer~\citep{MEAformer}, UMAEA~\citep{UMAEA}, PMF~\citep{PMF}, and HHEA~\citep{SimpleHHEA}. 
Following~\citep{MEAformer,PMF,XGEA}, we conduct experiments under two widely-adopted evaluation protocols: \textit{Non-name setting} denotes all attributes except for the entity name are used, while \textit{All-attributes setting} includes all available modalities.
For fair comparisons, we adopt the same backbone (\textit{i.e.}, CLIP) for all baselines and our method.
For more results on different backbones, please refer to Appendix~\ref{appendix: various_backbones}.

\begin{figure}[t]
    \centering
    \includegraphics[width=1.0\linewidth]{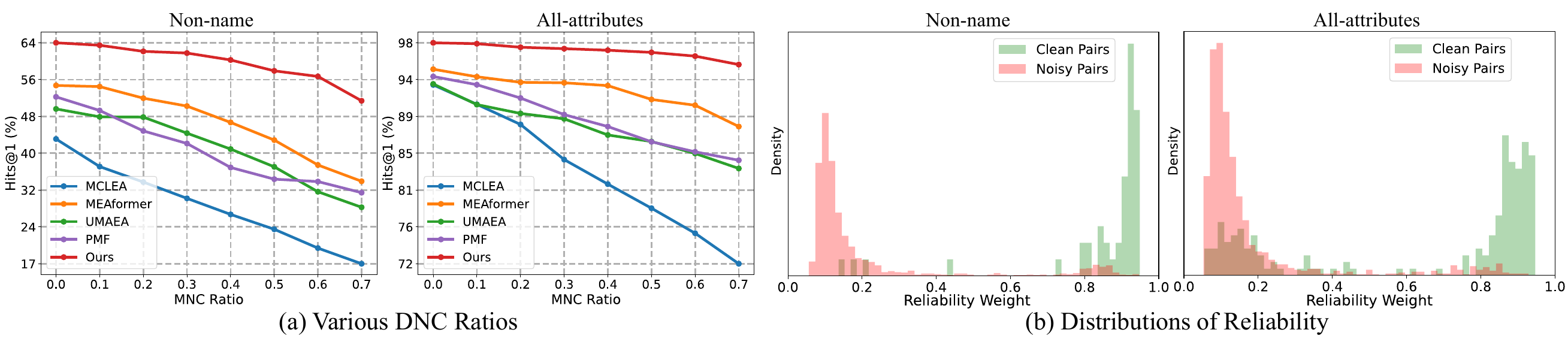}
    \caption{Analytic studies of various DNC ratios and reliability in Eq.~\ref{eq: reliablity_weights}.}
    \label{fig: finer-grained ablation}
\end{figure}

\begin{table*}[t]
    \centering
    \caption{Comparisons with state-of-the-art methods on All-attributes benchmarks under DNC setting. }
    \label{tab: main table textual}
    \resizebox{0.95\linewidth}{!}{  
    \begin{tabular}{c|l|ccc|ccc|ccc|ccc|ccc|c}
        \toprule
        \multirow{2}{*}{Setting} & \multirow{2}{*}{Method} & \multicolumn{3}{c|}{ICEWS-WIKI} & \multicolumn{3}{c|}{ICEWS-YAGO} & \multicolumn{3}{c|}{DBP15K$_{\text{ZH-EN}}$} & \multicolumn{3}{c|}{DBP15K$_{\text{JA-EN}}$} & \multicolumn{3}{c}{DBP15K$_{\text{FR-EN}}$} & \multirow{2}{*}{\makecell{Avg.\\H@1}}  \\
                                 &                         & H@1 & H@5 & MRR & H@1 & H@5 & MRR & H@1 & H@5 & MRR & H@1 & H@5 & MRR & H@1 & H@5 & MRR \\
        \midrule
        \multirow{8}{*}{\parbox[c][1cm][c]{1cm}{\centering Inherent \\ DNC}} 
        & EVA         & 90.7 & 95.7 & 93.0 & 86.5 & 94.0 & 89.8 & 89.8 & 96.6 & 92.8 & 94.8 & 98.8 & 96.5 & 98.7 & 99.8 & 99.2 & 92.1 \\
        & MCLEA       & 93.8 & 98.3 & 95.9 & 92.1 & 97.7 & 94.6 & 94.5 & 98.6 & 96.4 & 97.8 & 99.7 & 98.7 & 99.2 & 99.9 & 99.5 & 95.5\\
        & XGEA        & 83.5 & 94.4 & 88.6 & \underline{93.9} & 97.3 & \underline{95.8} & 91.4 & 97.4 & 94.1 & 94.3 & 98.0 & 96.0 & 97.3 & 99.3 & 98.2 & 92.1 \\
        & MEAformer   & \underline{95.9} & \underline{98.8} & \underline{97.2} & 93.8 & \underline{97.9} & 95.7 & \underline{96.7} & 99.0 & \underline{97.7} & \underline{98.8} & \underline{99.8} & \underline{99.3} & \underline{99.6} & \textbf{100.0} & \underline{99.8} & \underline{97.0}\\
        & UMAEA       & 94.8 & 98.7 & 96.6 & 92.8 & \underline{97.9} & 95.1 & 95.4 & 98.9 & 97.0 & 98.2 & 99.7 & 98.9 & 99.4 & 99.9 & 99.6 & 96.1\\
        & PMF         & 94.9 & 98.4 & 96.5 & 92.8 & 97.7 & 95.0 & 96.3 & \underline{99.1} & 97.6 & 98.5 & 99.7 & 99.1 & 99.5 & \textbf{100.0} & 99.7 & 96.4\\
        & HHEA & 89.9 & 95.5 & 92.5 & 89.7 & 95.2 & 92.2 & 68.1 & 78.8 & 73.2 & 77.0 & 86.0 & 81.1 & 85.8 & 92.2 & 88.7 & 82.1\\
        & Ours        & \textbf{98.9} & \textbf{99.2} & \textbf{99.1} & \textbf{97.6} & \textbf{98.8} & \textbf{98.2} & \textbf{98.3} & \textbf{99.5} & \textbf{98.8} & \textbf{99.3} & \textbf{99.9} & \textbf{99.6} & \textbf{99.8} & \textbf{100.0} & \textbf{99.9} & \textbf{98.8}\\
        \midrule
        \multirow{8}{*}{\parbox[c][1cm][c]{1cm}{\centering 20\% \\ DNC}}
        & EVA         & 67.4 & 76.2 & 71.6 & 17.9 & 21.4 & 19.7 & 64.2 & 78.9 & 70.8 & 72.6 & 85.7 & 78.5 & 88.0 & 95.2 & 91.3 & 62.0\\
        & MCLEA       & 89.0 & 95.2 & 91.8 & 88.8 & 95.8 & 92.0 & 91.5 & 97.0 & 94.0 & 95.6 & 98.8 & 97.0 & 97.8 & 99.6 & 98.6 & 92.5\\
        & XGEA        & 56.1 & 67.3 & 61.7 & 60.1 & 71.4 & 65.5 & 89.5 & 96.4 & 92.6 & 92.4 & 98.2 & 95.0 & 96.6 & 98.9 & 97.6 & 78.9\\
        & MEAformer   & \underline{93.8} & \underline{97.6} & \underline{95.6} & \underline{91.8} & \underline{97.2} & \underline{94.3} & \underline{95.5} & \underline{98.5} & \underline{96.8} & \underline{98.3} & \underline{99.6} & \underline{98.9} & \underline{99.4} & \underline{99.9} & \underline{99.7} & \underline{95.7}\\
        & UMAEA       & 90.3 & 96.5 & 93.1 & 86.8 & 95.1 & 90.5 & 94.1 & 98.2 & 95.9 & 97.2 & 99.4 & 98.2 & 98.8 & \underline{99.9} & 99.3 & 93.5\\
        & PMF         & 92.2 & 96.9 & 94.3 & 90.9 & 96.3 & 93.4 & 94.8 & 98.1 & 96.3 & 97.6 & 99.3 & 98.3 & 99.2 & \underline{99.9} & 99.5 & 94.9\\
        & HHEA & 87.6 & 93.8 & 90.5 & 89.3 & 94.6 & 92.1 & 66.1 & 75.9 & 70.8 & 72.6 & 81.9 & 77.0 & 83.5 & 90.2 & 86.6 & 79.8\\
        & Ours        & \textbf{98.3} & \textbf{98.9} & \textbf{98.6} & \textbf{97.5} & \textbf{98.7} & \textbf{98.1} & \textbf{97.6} & \textbf{99.1} & \textbf{98.3} & \textbf{99.1} & \textbf{99.9} & \textbf{99.5} & \textbf{99.8} & \textbf{100.0} & \textbf{99.9} & \textbf{98.5}\\
        \midrule
        \multirow{8}{*}{\parbox[c][1cm][c]{1cm}{\centering 50\% \\ DNC}}
        & EVA         & 2.7 & 3.8 & 3.4 & 0.0 & 0.1 & 0.2 & 17.5 & 31.7 & 24.2 & 18.4 & 33.2 & 25.2 & 15.3 & 30.5 & 22.4 & 10.8\\
        & MCLEA       & 78.9 & 88.3 & 83.2 & 75.9 & 88.1 & 81.5 & 84.5 & 91.7 & 87.8 & 88.7 & 94.7 & 91.4 & 93.5 & 97.5 & 95.4 & 84.3\\
        & XGEA        & 50.3 & 60.3 & 55.3 & 34.8 & 44.5 & 39.8 & 71.3 & 86.4 & 78.0 & 70.1 & 85.5 & 77.0 & 88.7 & 95.9 & 91.9 & 63.0\\
        & MEAformer   & \underline{91.9} & \underline{96.7} & \underline{94.1} & \underline{91.9} & \underline{96.8} & \underline{94.1} & \underline{93.4} & \underline{97.3} & \underline{95.2} & \underline{97.3} & \underline{99.1} & \underline{98.1} & \underline{99.1} & \underline{99.9} & \underline{99.5} & \underline{94.7}\\
        & UMAEA       & 87.0 & 94.4 & 90.4 & 85.7 & 93.9 & 89.4 & 91.4 & 96.7 & 93.8 & 95.9 & 98.8 & 97.2 & 98.1 & 99.6 & 98.8 & 91.6\\
        & PMF         & 86.9 & 93.9 & 90.0 & 87.6 & 94.4 & 90.7 & 92.2 & 96.5 & 94.2 & 96.1 & 98.8 & 97.3 & 98.6 & 99.6 & 99.1 & 92.3\\
        & HHEA & 86.2 & 92.8 & 89.2 & 84.2 & 92.1 & 87.8 & 56.8 & 71.3 & 63.7 & 70.5 & 82.2 & 75.9 & 76.9 & 86.1 & 81.1 & 74.9\\
        & Ours        & \textbf{97.7} & \textbf{98.3} & \textbf{98.0} & \textbf{97.0} & \textbf{98.2} & \textbf{97.6} & \textbf{96.3} & \textbf{98.1} & \textbf{97.2} & \textbf{98.7} & \textbf{99.7} & \textbf{99.1} & \textbf{99.7} & \textbf{100.0} & \textbf{99.8} & \textbf{97.9}\\
        \bottomrule
    \end{tabular}
    }
\end{table*}

As shown in Tables~\ref{tab: main table visual}-\ref{tab: main table textual}, we could have the following conclusions:
i) existing methods face substantial performance degradation as noise increases, highlighting their vulnerability to noisy correspondences. 
In contrast, RULE outperforms all baselines across different datasets and noise settings, demonstrating superior robustness against DNC;
ii) even without any manually-injected noise, RULE still achieves performance gains compared to existing methods, as the real-world MMEA datasets contain a considerable number of DNC.
To further verify the effectiveness of RULE, we conduct experiments under the manually-injected noise ratio from $0.0$ to $0.7$. As shown in Fig.~\ref{fig: finer-grained ablation} (a), RULE not only achieves higher performance across all noise levels but also exhibits significantly slower performance degradation, which further confirms the robustness of RULE against DNC.

\subsection{Analysis and Ablation Study}
\label{sec: Ablation and Analytic Study}
In this section, we conduct analysis and ablation studies on the ICEWS-WIKI dataset.

\begin{figure}[htbp]
    \centering
    \begin{minipage}{0.33\linewidth}
        \centering
        \includegraphics[width=0.9\linewidth]{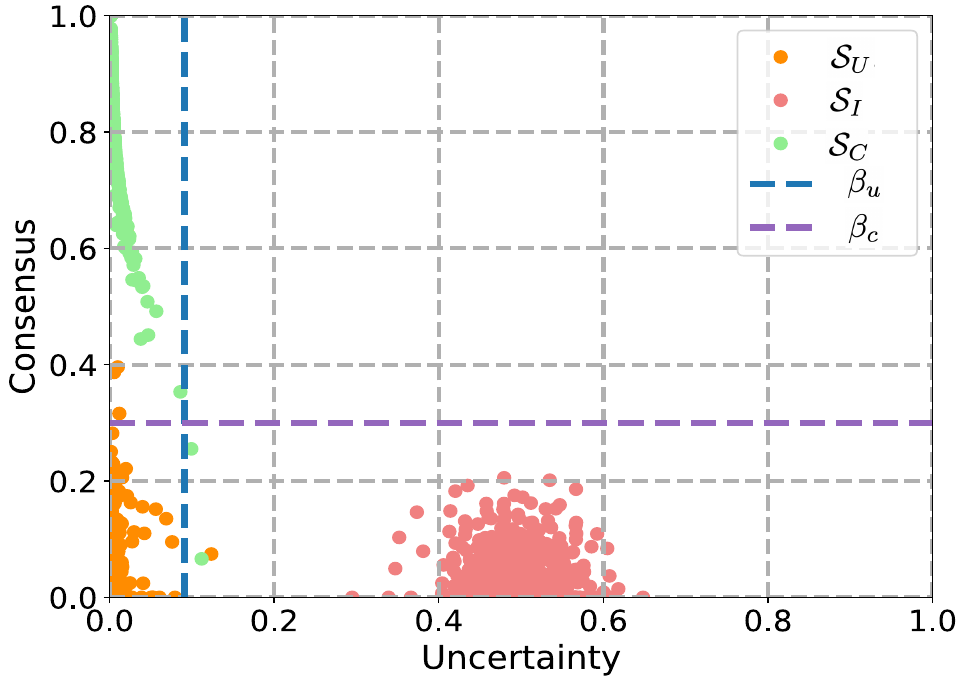}
        \vspace{-1em} 
        \caption{Quantitative analysis of the uncertainty and consensus on the name attribute.}
        \label{fig: three_pairs}
    \end{minipage}
    \hspace{0.2cm}  
    \begin{minipage}{0.6\linewidth}
        \centering
        \captionof{table}{Ablation study of the various modules in the train and test stages.}
        \label{tab: ablation module}
        \vspace{-0.5em} 
        \resizebox{0.9\linewidth}{!}{
        \begin{tabular}{c|l|ccc|ccc}
        \toprule
        \multirow{2}{*}{Stage} & \multirow{2}{*}{Setting} & \multicolumn{3}{c|}{\textbf{Non-name}} & \multicolumn{3}{c}{\textbf{All-attributes}} \\
         & & H@1 & H@5 & MRR & H@1 & H@5 & MRR \\
        \midrule
        \multirow{4}{*}{Train} 
        & w/o DRL        & 31.6 & 45.9 & 38.6 & 82.3 & 90.4 & 86.0 \\
        & w/o DRF        & 50.4 & 66.2 & 57.6 & 93.4 & 97.4 & 95.2 \\
        & Only Unc.      & 53.5 & 67.8 & 60.2 & 93.6 & 97.4 & 95.4 \\
        & Only Cons.     & 48.3 & 60.3 & 54.3 & 87.7 & 93.2 & 90.4 \\
        \midrule
        \multirow{3}{*}{Test} 
        & w/o DRF        & 52.4 & 66.2 & 59.0 & 95.1 & 97.9& 96.3 \\
        & w/o TTR        & 56.5 & 68.6 & 62.3 & 94.0 & 97.7 & 95.7 \\
         & MLLM Enhance   & 56.6 & 69.0 & 62.4 & 97.6 & 98.2 & 97.9 \\
        \midrule
        Both & Default    & 58.2 & 69.7 & 63.6 & 97.7 & 98.3 & 98.0 \\
        \bottomrule
        \end{tabular}
        }
    \end{minipage}
\end{figure}

\textbf{Analysis Studies on Uncertainty and Consensus.}
As discussed in Section~\ref{sec: Reliability Estimation and Pair Division}, the estimated reliability plays a key role in identifying DNC. To better understand its behavior, we visualize the reliability distribution of all training entity pairs.  
As shown in Fig.~\ref{fig: finer-grained ablation}(b), clean pairs are concentrated on the right side of the plot (indicating high reliability), while noisy pairs are predominantly on the left (indicating low reliability). This confirms that the proposed reliability serves an effective indicator for distinguishing clean and noisy pairs.
To further explore how the proposed uncertainty and consensus behave under noise, we construct subsets $\mathcal{S}_{U}$ and $\mathcal{S}_{I}$ by injecting synthetic noise and randomly shuffling the name attributes of the raw set $\mathcal{S}_{C}$. As illustrated in Fig.~\ref{fig: three_pairs}, uncertainty and consensus principles successfully separate the three subsets, which supports the design of our tailored loss strategies in Eq.~\ref{eq: dually robust loss}.

\textbf{Effectiveness of Robust Fusion}. 
To qualitatively study the effectiveness of RULE in handling entity-attribute noise, we visualize the reliability in Eq.~\ref{eq: reliablity_weights} during the fusion process.
As shown in Fig.\ref{fig: visualization}, correctly associated attributes are assigned high reliability scores, while noisy or irrelevant attributes receive significantly lower scores. This behavior confirms that RULE effectively suppresses the influence of unreliable attributes during fusion, thereby enhancing robustness against entity-attribute noise.

\textbf{Ablation studies.}
To verify the effectiveness of each component in our framework, we conduct ablation experiments on the modules involved in both training and test-time phases. 
According to the results in Table~\ref{tab: ablation module}, one could have the following conclusions.
First, during training phase, both the ``Only Unc.'' variant (which applies the uncertainty-guided loss in Eq.~\ref{eq: evi loss}) and the ``Only Cons.'' variant (which uses a consensus-based MSE loss) outperform the baseline ``w/o DRL'' (which uses only a standard MSE loss). This demonstrates the effectiveness of our proposed Dually Robust Learning mechanism in handling noisy correspondence.
Second, during the test phase, the TTR module significantly improves alignment performance by uncovering latent semantic connections. In particular, the comparison between the ``MLLM Enhance'' (which only uses rethinking scores in Eq.~\ref{eq: mllm reasoning}) and ``w/o TTR'' settings shows that combining rethinking scores with prior similarity scores leads to complementary effects, resulting in improved robustness and accuracy.
Third, the Dually Robust Fusion (DRF) module effectively mitigates the influence of intra-entity NC. Its inclusion enhances performance in both the training and testing stages.

\begin{figure}[t]
    \centering
    \includegraphics[width=0.95\linewidth]{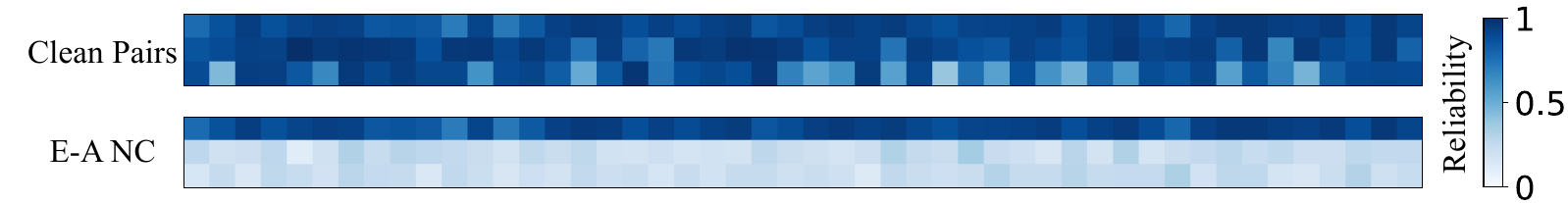}
    \caption{Visualization of the reliability for clean entity pairs and those with manually injected E-A NC in the image and name attributes. From top to bottom, the three rows denote the structure, image, and name attributes, respectively.}
    \label{fig: visualization}
\end{figure}

\section{Conclusion}
In this paper, we study a new problem in MMEA, \textit{i.e.}, Dual-level Noisy Correspondence, which refers to the wrongly annotated intra-entity and inter-graph correspondences.
To solve this problem, the proposed methods estimate the reliability of both the intra-entity and inter-graph correspondences and alleviate the negative impact of NC during the inter-graph discrepancy elimination and intra-entity attribution fusion.
Beyond the training-time design, we employ a novel correspondence reasoning module to guarantee more accurate equivalent entity identification during inference.
We believe this work might remarkably enrich the learning
paradigm with noisy correspondence by simultaneously considering the noise across both training-time and test-time.



\bibliography{neurips_2025}

\begin{thebibliography}{39}
\providecommand{\natexlab}[1]{#1}
\providecommand{\url}[1]{\texttt{#1}}
\expandafter\ifx\csname urlstyle\endcsname\relax
  \providecommand{\doi}[1]{doi: #1}\else
  \providecommand{\doi}{doi: \begingroup \urlstyle{rm}\Url}\fi

\bibitem[Bai et~al.(2025)Bai, Chen, Liu, Wang, Ge, Song, Dang, Wang, Wang, Tang, et~al.]{QWEN2.5}
Shuai Bai, Keqin Chen, Xuejing Liu, Jialin Wang, Wenbin Ge, Sibo Song, Kai Dang, Peng Wang, Shijie Wang, Jun Tang, et~al.
\newblock Qwen2.5-vl technical report.
\newblock \emph{arXiv preprint arXiv:2502.13923}, 2025.

\bibitem[Brame(2016)]{active}
Cynthia Brame.
\newblock Active learning.
\newblock \emph{Vanderbilt University Center for Teaching}, 2016.

\bibitem[Chen et~al.(2020)Chen, Kornblith, Norouzi, and Hinton]{simclr}
Ting Chen, Simon Kornblith, Mohammad Norouzi, and Geoffrey Hinton.
\newblock A simple framework for contrastive learning of visual representations.
\newblock In \emph{ICML}, 2020.

\bibitem[Chen et~al.(2023{\natexlab{a}})Chen, Chen, Zhang, Guo, Fang, Huang, Zhang, Geng, Pan, Song, et~al.]{MEAformer}
Zhuo Chen, Jiaoyan Chen, Wen Zhang, Lingbing Guo, Yin Fang, Yufeng Huang, Yichi Zhang, Yuxia Geng, Jeff~Z Pan, Wenting Song, et~al.
\newblock Meaformer: Multi-modal entity alignment transformer for meta modality hybrid.
\newblock In \emph{ACM Multimedia}, 2023{\natexlab{a}}.

\bibitem[Chen et~al.(2023{\natexlab{b}})Chen, Guo, Fang, Zhang, Chen, Pan, Li, Chen, and Zhang]{UMAEA}
Zhuo Chen, Lingbing Guo, Yin Fang, Yichi Zhang, Jiaoyan Chen, Jeff~Z Pan, Yangning Li, Huajun Chen, and Wen Zhang.
\newblock Rethinking uncertainly missing and ambiguous visual modality in multi-modal entity alignment.
\newblock In \emph{ISWC}, 2023{\natexlab{b}}.

\bibitem[Gong et~al.(2021)Gong, Huang, and Chen]{multi_modal_deviation}
Yunpeng Gong, Liqing Huang, and Lifei Chen.
\newblock Eliminate deviation with deviation for data augmentation and a general multi-modal data learning method.
\newblock \emph{arXiv preprint arXiv:2101.08533}, 2021.

\bibitem[Gong et~al.(2022)Gong, Huang, and Chen]{cross_modal_attack2}
Yunpeng Gong, Liqing Huang, and Lifei Chen.
\newblock Person re-identification method based on color attack and joint defence.
\newblock In \emph{CVPR, 2022}, pp.\  4313--4322, 2022.

\bibitem[Gong et~al.(2024)Gong, Zhong, Qu, Luo, Ji, and Jiang]{cross_modal_attack1}
Yunpeng Gong, Zhun Zhong, Yansong Qu, Zhiming Luo, Rongrong Ji, and Min Jiang.
\newblock Cross-modality perturbation synergy attack for person re-identification.
\newblock \emph{Advances in Neural Information Processing Systems}, 37:\penalty0 23352--23377, 2024.

\bibitem[Guo et~al.(2021)Guo, Tang, Zeng, Zhao, and Liu]{HMEA}
Hao Guo, Jiuyang Tang, Weixin Zeng, Xiang Zhao, and Li~Liu.
\newblock Multi-modal entity alignment in hyperbolic space.
\newblock \emph{Neurocomputing}, 2021.

\bibitem[Han et~al.(2022)Han, Zhang, Fu, and Zhou]{TMC}
Zongbo Han, Changqing Zhang, Huazhu Fu, and Joey~Tianyi Zhou.
\newblock Trusted multi-view classification with dynamic evidential fusion.
\newblock \emph{IEEE Transactions on Pattern Analysis and Machine Intelligence}, 2022.

\bibitem[Huang et~al.(2025)Huang, Zhang, and Li]{huang2025enhance}
Sida Huang, Hongyuan Zhang, and Xuelong Li.
\newblock Enhance vision-language alignment with noise.
\newblock In \emph{AAAI}, 2025.

\bibitem[Huang et~al.(2024{\natexlab{a}})Huang, Zhang, Zhang, Chen, and Kim]{PMF}
Yani Huang, Xuefeng Zhang, Richong Zhang, Junfan Chen, and Jaein Kim.
\newblock Progressively modality freezing for multi-modal entity alignment.
\newblock \emph{arXiv preprint arXiv:2407.16168}, 2024{\natexlab{a}}.

\bibitem[Huang et~al.(2021)Huang, Niu, Liu, Ding, Xiao, Wu, and Peng]{NCR}
Zhenyu Huang, Guocheng Niu, Xiao Liu, Wenbiao Ding, Xinyan Xiao, Hua Wu, and Xi~Peng.
\newblock Learning with noisy correspondence for cross-modal matching.
\newblock In \emph{NeurIPS}, 2021.

\bibitem[Huang et~al.(2024{\natexlab{b}})Huang, Yang, Xiao, Hu, and Peng]{Never}
Zhenyu Huang, Mouxing Yang, Xinyan Xiao, Peng Hu, and Xi~Peng.
\newblock Noise-robust vision-language pre-training with positive-negative learning.
\newblock \emph{IEEE Transactions on Pattern Analysis and Machine Intelligence}, 2024{\natexlab{b}}.

\bibitem[Jiang et~al.(2024)Jiang, Xu, Shen, Wang, Su, Shi, Sun, Li, Guo, and Shen]{SimpleHHEA}
Xuhui Jiang, Chengjin Xu, Yinghan Shen, Yuanzhuo Wang, Fenglong Su, Zhichao Shi, Fei Sun, Zixuan Li, Jian Guo, and Huawei Shen.
\newblock Toward practical entity alignment method design: Insights from new highly heterogeneous knowledge graph datasets.
\newblock In \emph{WWW}, 2024.

\bibitem[Li et~al.(2022)Li, Li, Xiong, and Hoi]{blip}
Junnan Li, Dongxu Li, Caiming Xiong, and Steven Hoi.
\newblock Blip: Bootstrapping language-image pre-training for unified vision-language understanding and generation.
\newblock In \emph{ICML}, 2022.

\bibitem[Li et~al.(2023)Li, Guo, Luo, Ji, Wang, Sheng, and Li]{ACK-MMEA}
Qian Li, Shu Guo, Yangyifei Luo, Cheng Ji, Lihong Wang, Jiawei Sheng, and Jianxin Li.
\newblock Attribute-consistent knowledge graph representation learning for multi-modal entity alignment.
\newblock In \emph{WWW}, 2023.

\bibitem[Lin et~al.(2022{\natexlab{a}})Lin, Chen, Wang, and Zhang]{uncoupled_clustering}
Jia-Qi Lin, Man-Sheng Chen, Chang-Dong Wang, and Haizhang Zhang.
\newblock A tensor approach for uncoupled multiview clustering.
\newblock \emph{IEEE Transactions on Cybernetics}, 2022{\natexlab{a}}.

\bibitem[Lin et~al.(2024)Lin, Chen, Zhu, Wang, and Zhang]{attribute_clustering}
Jia-Qi Lin, Man-Sheng Chen, Xi-Ran Zhu, Chang-Dong Wang, and Haizhang Zhang.
\newblock Dual information enhanced multiview attributed graph clustering.
\newblock \emph{IEEE Transactions on Neural Networks and Learning Systems}, 2024.

\bibitem[Lin et~al.(2023)Lin, Yang, Yu, Hu, Zhang, and Peng]{COMMON}
Yijie Lin, Mouxing Yang, Jun Yu, Peng Hu, Changqing Zhang, and Xi~Peng.
\newblock Graph matching with bi-level noisy correspondence.
\newblock In \emph{ICCV}, 2023.

\bibitem[Lin et~al.(2022{\natexlab{b}})Lin, Zhang, Wang, Shi, Wu, and Zheng]{MCLEA}
Zhenxi Lin, Ziheng Zhang, Meng Wang, Yinghui Shi, Xian Wu, and Yefeng Zheng.
\newblock Multi-modal contrastive representation learning for entity alignment.
\newblock \emph{arXiv preprint arXiv:2209.00891}, 2022{\natexlab{b}}.

\bibitem[Liu et~al.(2021)Liu, Chen, Roth, and Collier]{EVA}
Fangyu Liu, Muhao Chen, Dan Roth, and Nigel Collier.
\newblock Visual pivoting for (unsupervised) entity alignment.
\newblock In \emph{AAAI}, 2021.

\bibitem[Liu et~al.(2023)Liu, Li, Wu, and Lee]{LLAVA}
Haotian Liu, Chunyuan Li, Qingyang Wu, and Yong~Jae Lee.
\newblock Visual instruction tuning.
\newblock In \emph{NeurIPS}, 2023.

\bibitem[Liu et~al.(2019)Liu, Li, Garcia-Duran, Niepert, Onoro-Rubio, and Rosenblum]{mmkg1}
Ye~Liu, Hui Li, Alberto Garcia-Duran, Mathias Niepert, Daniel Onoro-Rubio, and David~S Rosenblum.
\newblock Mmkg: Multi-modal knowledge graphs.
\newblock In \emph{ESWC}, 2019.

\bibitem[Natarajan et~al.(2013)Natarajan, Dhillon, Ravikumar, and Tewari]{noisylabel}
Nagarajan Natarajan, Inderjit~S Dhillon, Pradeep~K Ravikumar, and Ambuj Tewari.
\newblock Learning with noisy labels.
\newblock In \emph{NeurIPS}, 2013.

\bibitem[Paulheim(2016)]{crowdsourcing}
Heiko Paulheim.
\newblock Knowledge graph refinement: A survey of approaches and evaluation methods.
\newblock \emph{Semantic web}, 2016.

\bibitem[Pei et~al.(2020)Pei, Yu, Yu, and Zhang]{REA}
Shichao Pei, Lu~Yu, Guoxian Yu, and Xiangliang Zhang.
\newblock Rea: Robust cross-lingual entity alignment between knowledge graphs.
\newblock In \emph{KDD}, 2020.

\bibitem[Piscopo \& Simperl(2018)Piscopo and Simperl]{wiki_intra_nc}
Alessandro Piscopo and Elena Simperl.
\newblock Who models the world? collaborative ontology creation and user roles in wikidata.
\newblock \emph{Proceedings of the ACM on Human-Computer Interaction}, 2\penalty0 (CSCW):\penalty0 1--18, 2018.

\bibitem[Radford et~al.(2021)Radford, Kim, Hallacy, Ramesh, Goh, Agarwal, Sastry, Askell, Mishkin, Clark, et~al.]{CLIP}
Alec Radford, Jong~Wook Kim, Chris Hallacy, Aditya Ramesh, Gabriel Goh, Sandhini Agarwal, Girish Sastry, Amanda Askell, Pamela Mishkin, Jack Clark, et~al.
\newblock Learning transferable visual models from natural language supervision.
\newblock In \emph{ICML}, 2021.

\bibitem[Sensoy et~al.(2018)Sensoy, Kaplan, and Kandemir]{Evidentiallearning}
Murat Sensoy, Lance Kaplan, and Melih Kandemir.
\newblock Evidential deep learning to quantify classification uncertainty.
\newblock In \emph{NeurIPS}, 2018.

\bibitem[Shafer(1992)]{Dempster–Shafer}
Glenn Shafer.
\newblock Dempster-shafer theory.
\newblock \emph{Encyclopedia of Artificial Intelligence}, 1992.

\bibitem[Suchanek et~al.(2007)Suchanek, Kasneci, and Weikum]{yago}
Fabian~M Suchanek, Gjergji Kasneci, and Gerhard Weikum.
\newblock Yago: A core of semantic knowledge.
\newblock In \emph{WWW}, 2007.

\bibitem[Vrandečić \& Krötzsch(2014)Vrandečić and Krötzsch]{wiki}
Denny Vrandečić and Markus Krötzsch.
\newblock Wikidata: A free collaborative knowledgebase.
\newblock \emph{Communications of the ACM}, 2014.

\bibitem[Xu et~al.(2023)Xu, Xu, and Su]{XGEA}
Baogui Xu, Chengjin Xu, and Bing Su.
\newblock Cross-modal graph attention network for entity alignment.
\newblock In \emph{ACM Multimedia}, 2023.

\bibitem[Xu et~al.(2024)Xu, Lu, Su, and Yan]{pmcea}
Baogui Xu, Yafei Lu, Bing Su, and Xiaoran Yan.
\newblock Position-aware active learning for multi-modal entity alignment.
\newblock In \emph{ICASSP}. IEEE, 2024.

\bibitem[Yang et~al.(2021)Yang, Li, Huang, Liu, Hu, and Peng]{MvCLN}
Mouxing Yang, Yunfan Li, Zhenyu Huang, Zitao Liu, Peng Hu, and Xi~Peng.
\newblock Partially view-aligned representation learning with noise-robust contrastive loss.
\newblock In \emph{CVPR}, 2021.

\bibitem[Yang et~al.(2022)Yang, Huang, Hu, Li, Lv, and Peng]{DART}
Mouxing Yang, Zhenyu Huang, Peng Hu, Taihao Li, Jiancheng Lv, and Xi~Peng.
\newblock Learning with twin noisy labels for visible-infrared person re-identification.
\newblock In \emph{CVPR}, 2022.

\bibitem[Zhai et~al.(2023)Zhai, Mustafa, Kolesnikov, and Beyer]{siglip}
Xiaohua Zhai, Basil Mustafa, Alexander Kolesnikov, and Lucas Beyer.
\newblock Sigmoid loss for language image pre-training.
\newblock In \emph{ICCV}, 2023.

\bibitem[Zhu et~al.(2022)Zhu, Li, Wang, Jiang, Sun, Wang, Xiao, and Yuan]{mmkg2}
Xiangru Zhu, Zhixu Li, Xiaodan Wang, Xueyao Jiang, Penglei Sun, Xuwu Wang, Yanghua Xiao, and Nicholas~Jing Yuan.
\newblock Multi-modal knowledge graph construction and application: A survey.
\newblock \emph{IEEE Transactions on Knowledge and Data Engineering}, 2022.

\end{thebibliography}
\bibliographystyle{neurips_2025}

\newpage
\appendix
\begin{leftline}
	{
		\LARGE{\textsc{Appendix}}
	}
\end{leftline}
	\etocdepthtag.toc{mtappendix}
    \etocsettagdepth{mtchapter}{none}
    \etocsettagdepth{mtappendix}{subsection}
    
    {
     
        \footnotesize
        \etocsettocstyle{}{}  
        \tableofcontents
    }


\newpage

\section{Noise Statistics in Real-World Benchmarks}
\label{sec: Noise Statistics in Real-World Benchmarks}
In this section, we elaborate on the necessity of addressing the DNC challenge in real-world scenarios and discuss the underlying reasons behind the formation of the DNC.
\subsection{Statistics Analysis}
As discussed in the manuscript, real-world benchmarks implicitly contain both intra-entity and inter-graph noisy correspondences. To support this claim, we conducted an additional observational study to analyze the types and distributions of noise present in the datasets. Specifically, we randomly sample 1,000 entity pairs from the training set in the ICEWS-WIKI and ICEWS-YAGO benchmarks and conduct manual statistical analysis. 
After that, we categorize each entity pair into one of four types, \textit{i.e.}, ``Clean'', ``E-E NC'', ``E-A NC'' or ``A-A NC''.
As shown in Fig.~\ref{fig: statistics}, one could have the following observation and conclusions:
\begin{itemize}
    \item Even after careful manual annotation, a considerable amount of DNC still exists in real-world datasets, \textit{i.e.}, more than 50\% of entity pairs in ICEWS benchmarks suffer from the DNC challenge.
    \item Among the three types of NC, entity-attribute NC and attribute-attribute NC account for a large proportion (\textit{e.g.}, over 40\% in ICEWS benchmarks).
    Compared with the entity-entity NC, these two types of NC are more fine-grained, which makes it difficult to uncover manually.
    It is worth emphasizing that although several NC-oriented methods have been proposed, nearly all of them overlook such non-negligible fine-grained noise in MMEA, further highlighting the necessity of addressing the DNC challenge.
\end{itemize}
As mentioned above, although the two commonly-used MMEA datasets are small in scale (\textit{i.e.}, containing up to 30K entity pairs) and carefully annotated by humans, they still exhibit non-negligible DNC cases. 
Such a DNC challenge would become even more severe in large-scale knowledge graphs such as Wikidata, DBpedia, and YAGO datasets (\textit{i.e.}, containing over 10M-100M entities).
We believe that the DNC challenge is one of the key reasons why large-scale datasets are currently lacking in the MMEA field. The proposed RULE framework offers a promising solution toward achieving entity alignment on larger-scale datasets with the DNC challenge.

\begin{figure}[htbp]
    \centering
    \includegraphics[width=0.8\linewidth]{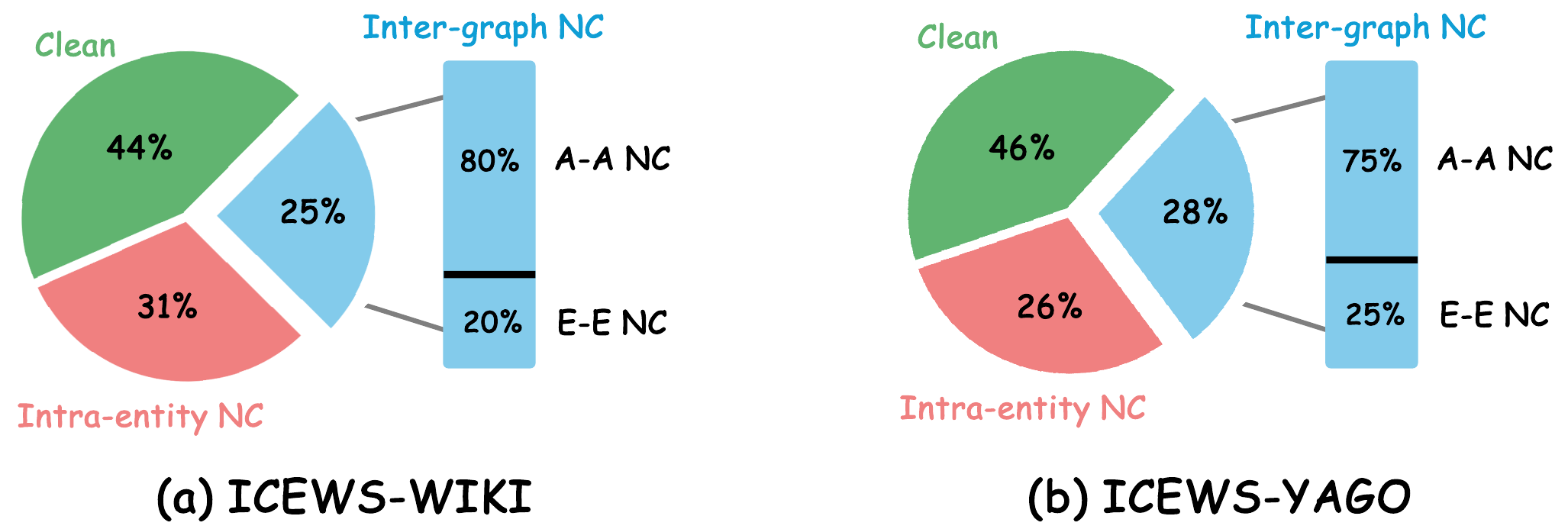}
    \caption{Distribution analysis of DNC noise in ICEWS-WIKI and ICEWS-YAGO benchmarks.}
    \label{fig: statistics}
\end{figure}

\subsection{The Underlying Reason Behind DNC}
To investigate the underlying causes of DNC, we analyze the formation of intra-entity NC and inter-graph NC from two perspectives, \textit{i.e.}, knowledge graph construction and cross-graph annotation.

\textbf{Knowledge Graph Construction: } The construction of existing knowledge graphs~\citep{wiki_intra_nc} (\textit{e.g.}, Wikidata, Freebase, and YAGO) heavily relies on crowdsourcing~\citep{crowdsourcing}, \textit{i.e.}, user editing on collaborative platforms.
However, as most users lack expert knowledge, their edits are prone to errors such as incorrect data entry and conceptual misunderstandings, thereby inevitably introducing incorrect attributes within an entity.
As a result, the editing errors in crowdsourcing would lead to intra-entity NC.

\textbf{Cross-graph Annotation:} 
In real-world scenarios, most MMEA datasets are annotated using semi-automated methods, which would inevitably introduce the inter-graph NC.
Specifically, to construct entity pairs between the ICEWS and WIKI knowledge graphs, \citet{SimpleHHEA} employs the name attributes from the ICEWS graph as queries and uses the Wikidata API to retrieve the most relevant matches. 
As a result, entity-entity pairs are established between the queries and the retrieved results. However, such a process inevitably introduces inter-graph E-E NC due to the following reasons:
i) some entities exist in ICEWS but do not appear in Wikidata, making it impossible to establish correct correspondences;
ii) even some entities have highly similar name attributes, they could be very obvious mismatches. For instance, as shown in Fig.~\ref{fig: fig1}(a) in the manuscript, the movie entity ``Mr. \& Mrs. Smith'' may be mistakenly aligned with the real-life actor couple ``Will Smith and Mrs. Smith.'' In this case, despite the similarity in names, the two refer to totally different entities with distinct multi-modal attributes. In other words, entities considered ``neighbors'' during cross-graph annotation (\textit{e.g.}, retrieved via name-based lookups), may not actually be relevant in real-world scenarios. Such fine-grained errors require expert knowledge to distinguish, and may be undetected even during manual filtering.
Both of the above reasons would lead to obvious mistakes, where the aligned entity pairs are totally unrelated.
It is worth noting that such obvious mistakes prove the reasonableness of constructing incorrect correspondences to simulate real-world noise in our experiments.

Moreover, once the inter-graph entity pairs are associated, their corresponding attributes could be treated as matched. 
Therefore, as the by-product of the entity-attribute and entity-entity correspondences establishment, the association between inter-graph attributes might be mislead due to the aforementioned two kinds of annotation errors. 
In other words, E-A NC and E-E NC would further propagate to inter-graph attribute pairs, leading to noisy A-A correspondences. 
Beyond explicit annotation errors in A-A pairs, some attributes may contain textual typos or visual ambiguity (\textit{e.g.}, noise under low-light conditions) in real-world scenarios, which would also lead to inter-graph A-A NC.



\newpage

\section{Related Work}
\label{sec: related}
In this section, we briefly review two topics related to this work, \textit{i.e.}, multi-modal entity alignment and learning with noisy correspondence.
\subsection{Multi-modal Entity Alignment}
Multi-modal entity alignment aims to eliminate the discrepancy between heterogeneous MMKGs, so that the equivalent entities from various MMKGs could be identified.
Towards achieving this goal, numerous MMEA approaches have been proposed, which typically involve the following two-stage pipeline, namely, fuse the intra-entity attributes to form the representation for entities and perform cross-graph alignment on the paired attributes and entities.
According to their primary focus, most existing approaches could be broadly grouped into two categories:
i) fusion-centric methods~\citep{MEAformer,PMF}, which assign different weights according to the importance of each attribute-specific attribute during the fusion stage.
ii) alignment-centric methods~\citep{ACK-MMEA,MCLEA}, which mitigate inter-graph discrepancy by maximizing the similarity between associated pairs while minimizing that of mismatched ones across MMKGs during the alignment stage.

Among the existing approaches, REA~\citep{REA} is the most relevant to our work, while having the following remarkable differences.
First, REA focuses on inter-graph uni-modality entity alignment, whereas our work tackles multi-modal entity alignment, which involves mitigating discrepancies not only between heterogeneous graphs but also across multi-modal attributes.
Second, REA only considers misalignment in entity-entity pairs, while our work comprehensively reveals and studies the noisy correspondence at dual levels, namely, both the intra-entity and inter-graph.
Third, unlike REA solely concentrates on training-time robustness, our method further improves test-time robustness, facilitating more precise cross-graph equivalent entity identification.

\subsection{Learning with Noisy Correspondence}
Noisy correspondence refers to inherently irrelevant or relevant samples that are wrongly regarded as associated (\textit{a.k.a}, false positive) or unassociated (\textit{a.k.a}, false negative), which is first revealed and studied by~\citep{NCR,MvCLN}.
Considering that numerous applications require paired data as input, including but not limited to visual instruction tuning~\citep{huang2025enhance}, vision-language pre-training~\citep{Never}, object re-identification~\citep{DART}, and graph matching~\citep{COMMON}, how to learn with noisy correspondence rooted in data pairs has emerged as a new research direction, drawing increasing attention from both academia and industry. 

In this paper, we focus on mitigating the negative impact of the noisy positive correspondence issue. 
Unlike most existing noisy correspondence studies that tackle the errors in correspondence of a specific level (e.g., image-to-sample or pixel-to-pixel), this work delves into the multi-modal entity alignment task and reveals the specific dual-level noisy correspondence (DNC) problem for the first time.
In brief, DNC refers to the noisy correspondence involving in the entity-attribute, entity-entity, and attribute-attribute pairs, which misleads the multi-modal fusion and cross-graph alignment processes.
Therefore, it is desirable to customize a specific approach for MMEA with the DNC problem.

\newpage

\section{Evidential Learning}
\label{appendix: evidential learning}
In this section, we provide more details about the dually robust objective function in Eq.~\ref{eq: overall}.

Following~\citep{Evidentiallearning,TMC}, the Dirichlet distribution is parameterized by $\boldsymbol{\alpha}_{i}=[\alpha_{i1},\alpha_{i2},\cdots,\alpha_{iK}]$ and the probability density function in Eq.~\ref{eq: dually robust loss} is given by,
\begin{equation}
    D(\boldsymbol{p}_{i} \mid \boldsymbol{\alpha}_{i})= \begin{cases}\frac{1}{B(\boldsymbol{\alpha}_{i})} \prod_{j=1}^{\tilde{N}} p_{ij}^{\alpha_{ij}-1} & \text { for } \boldsymbol{p}_{i} \in \mathcal{S}^{\tilde{N}}, \\ 0 & \text { otherwise },\end{cases}
\end{equation}
where $B(\boldsymbol{\alpha}_{i})$ is the multinomial beta function, and $\mathcal{S}^{\tilde{N}}$ is the $K$-dimensional unit simplex.

With the above derivation, MMEA aims to guide the query probability $\boldsymbol{p}_{i}$ to approach the annotated correspondence $\boldsymbol{y}_{i}$. 
To achieve this, the uncertainty-based loss could be formulated as follows,
\begin{equation}
    \begin{aligned}
        \mathcal{L}_{U}(\boldsymbol{\alpha}_{i},\boldsymbol{y}_{i})&=\int\left\|\boldsymbol{y}_i-\boldsymbol{p}_i\right\|_2^2 \frac{1}{B\left(\boldsymbol{\alpha}_i\right)} \prod_{j=1}^{\tilde{N}} p_{i j}^{\alpha_{i j}-1} d \mathbf{p}_i \\
        &= \sum_{j=1}^{\tilde{N}} \left[\left(y_{i j}-\mathbb{E}\left[p_{i j}\right]\right)^2+\operatorname{Var}\left(p_{i j}\right)\right]\\
        & = \sum_{j=1}^{\tilde{N}} \left(y_{i j}-\frac{\alpha_{i j}}{Q_i}\right)^2+\frac{\alpha_{i j}\left(Q_i-\alpha_{i j}\right)}{Q_i^2\left(Q_i+1\right)}.
    \end{aligned}
    \label{eq: evi loss}
\end{equation}
where $\mathbb{E}\left[p_{i j}\right]$ and $\operatorname{Var}\left(p_{i j}\right)$  are the expected value and the variance of $p_{ij}$, respectively.
Following~\citep{Evidentiallearning}, the expected probability $\mathbb{E}\left[p_{i j}\right]$ could be estimated by $\frac{\alpha_{i j}}{Q_i}$.
Intuitively, Eq.~\ref{eq: evi loss} encourages higher evidence for correct correspondence compared to mismatched ones, but also prevents excessive optimization when the overall evidence is limited.
In other words, the probability $\frac{\alpha_{i j}}{Q_i}$ is proportional to the total evidence $Q_i$, thus limiting overconfidence in noisy correspondences.
Specifically,

\begin{theorem}
    The uncertainty-aware probability $\frac{\alpha_{ij}}{Q_i}$ is upper bounded by $\frac{Q_i - K + 1}{Q_i}$, which is proportional to $Q_{i}$, \textit{i.e.},
    \begin{equation}
        \frac{\alpha_{ij}}{Q_i} \leq \frac{Q_i - K + 1}{Q_i} \propto Q_i.
    \end{equation}
    \label{pro: uncertainty upper bound}
\end{theorem}

\begin{proof}
    According to the definition of evidence in Eq.~\ref{eq: evi}, each correspondence satisfies \(\alpha_{il} \geq 1\) for all \(l \neq j\). Therefore, $Q_{i}$ is lower-bounded by
    \begin{equation}
        Q_i = \alpha_{ij} + \sum_{l \neq j} \alpha_{il} \geq \alpha_{ij} + (K - 1).
    \end{equation}
    Rearranging the inequality yields,
    \begin{equation}
        \alpha_{ij} \leq Q_i - (K - 1).
    \end{equation}
    Then, dividing both sides by $Q_{i}$, we obtain the desired upper bound,
    \begin{equation}
    \frac{\alpha_{ij}}{Q_i} \leq \frac{Q_i - K + 1}{Q_i},
    \end{equation}
    where the upper bound is proportional to $Q_{i}$, \textit{i.e.},$\frac{Q_i - K + 1}{Q_i} \propto Q_i$.
\end{proof}

\newpage

\section{Proof of Theorem~\ref{theorem: wrong uncertainty}}
\label{appendix: theorem}
In this section, we present detailed proofs for Theorem~\ref {theorem: wrong uncertainty} in the main paper.
\begin{theorem}
A low uncertainty $u_i$ does not necessarily imply that the highest belief is assigned to the annotated correspondence, \textit{i.e.}, $y_{ij}=1$,
\begin{equation}
    z_i \text{ with low } u_i \centernot\Rightarrow \arg\max\, \boldsymbol{b}_{i} = \arg\max \boldsymbol{y}_{i}.
\end{equation}
\end{theorem}
\begin{proof}
We will prove the theorem by contradiction. Suppose that for \textit{all} evidence vectors with low uncertainty $u_i$, the highest belief is assigned to the annotated correspondence, i.e., 
\begin{equation}
    \forall  \boldsymbol{e}_i \text{ with low } u_i, \ \arg\max\, \boldsymbol{b}_{i} = \arg\max \boldsymbol{y}_{i}.
\end{equation}
Let us now consider the two evidence vectors $\boldsymbol{e}_i^{(1)}$ and $\boldsymbol{e}_i^{(2)}$ defined as:
\begin{equation}
    e_{ij}^{(1)} = 
\begin{cases}
Q_i, & \text{if } j = j_1, \\
0, & \text{otherwise},
\end{cases}
\qquad
e_{ij}^{(2)} = 
\begin{cases}
Q_i, & \text{if } j = j_2, \\
0, & \text{otherwise},
\end{cases}
\end{equation}
where $j_1 \ne j_2$, and $Q_i > 0$ is fixed. According to Eq.~\ref{eq: uncertainty}, both evidence vectors yield the same total evidence $Q_i$ and hence the same uncertainty $u_i$.
However, $\boldsymbol{e}_i^{(1)}$ suggests $(i,j_1)$ is the most probable correspondence, while $\boldsymbol{e}_i^{(2)}$ suggests $(i,j_2)$ as the most plausible option.
However, only one of them might indicate the correct correspondence.
Such an example contradicts the assumption that low uncertainty invariably leads to the highest belief being assigned to the annotated correspondence $j$. Consequently, the initial assumption is invalid.
\end{proof}

\section{More Implementation Details}
\label{appendix: experiments}
\subsection{Datasets}
\label{appedix: datasets}
We place the detailed statistics of the five MMEA datasets used in our experiments in Table~\ref{table: datasets}.
Note that, not every entity is paired with image-type attributes.

\begin{table}[ht]
\centering
\caption{Details of MMEA datasets. ``E-E pairs'' indicates the number of cross-graph entity-entity correspondences.}
\label{table: datasets}
\resizebox{0.95\linewidth}{!}{  
\begin{tabular}{l|l|cccccc}
\hline
\textbf{Dataset} & \textbf{KG} & \textbf{Entity} & \textbf{Triples} & \textbf{Numeric} & \textbf{Image} & \textbf{E-E pairs} \\
\hline
DBP15K$_{\text{ZH-EN}}$ & ZH (Chinese) & 19388 & 70414 & 8111 & 15912 & 15000 \\
                        & EN (English) & 19572 & 95142 & 7173 & 14125 & 15000 \\
\hline
DBP15K$_{\text{JA-EN}}$ & JA (Japanese) & 19814 & 77214 & 5882 & 12739 & 15000 \\
                        & EN (English) & 19600 & 93484 & 6066 & 13741 & 15000 \\
\hline
DBP15K$_{\text{FR-EN}}$ & FR (French) & 19661 & 105998  & 4547 & 10599 & 15000 \\
                        & EN (English) & 19993 & 115722 & 6422 & 13858 & 15000 \\
\hline
ICEWS-WIKI        & ICEWS & 11047 & 3527881 & - & 33341 & 5058 \\
                        & WIKI  & 15896 & 198257 & - & 47688 & 5058 \\
\hline
ICEWS-YAGO          & ICEWS & 26863 & 4192555  & - & 80589 & 18824 \\
                        & YAGO  & 22734 & 107118  & - & 68202 & 18824 \\
\hline
\end{tabular}
}
\end{table}

\subsection{Networks}
In this section, we describe the network architectures used for encoding multi-modal attributes.  
For the name and image attributes, we first employ the pre-trained CLIP ViT-L/14~\citep{CLIP} model to obtain the features. 
Note that, all the experiments are conducted on the same visual and textual features for fair comparison.
Following~\citep{MEAformer}, we employ trainable fully connected layers to project the image and name attributes into embedding vectors.
Missing image attributes are handled by assigning zero vectors as their embeddings.

For the DBP15K benchmarks, the multi-modal attributes include structure, relation, numerical value, image, name, and character attributes.
Following~\citep{XGEA}, we adopt cross-modal graph networks to encode the structure and relation attributes.
For the numerical attribute, we employ the Bag-of-Words (BoW) model to encode the relations as fixed-length vectors.
Besides, the character attributes are derived from character bigrams of the translated entity names.

For the ICEWS benchmarks, the multi-modal attributes include structure, image, and name attributes. Following~\citep{MEAformer}, we utilize Graph Attention Networks GAT to reduce the computational cost of encoding structural information, as the number of relation triplets in the ICEWS benchmark is significantly larger than that of the DBP15K benchmark.

\subsection{Details of Dually Robust Fusion Module}
\label{appendix: DRF}
To achieve robustness against intra-entity E-A NC, we propose a dually robust fusion that assigns low weights to unreliable attributes.
Here, we provide more details on how to construct the initial subset for the greedy strategy in Eq.~\ref{eq: greedy}. 
Although some attributes are often weakly associated or even irrelevant to the entity, we assume that most attributes should be correctly associated, otherwise, the cross-graph correspondence would not hold.
Based on this assumption, we adopt the following strategy for constructing the initial subset $\pi_0$,
\begin{equation}
|\pi_0| \geq 
    \begin{cases}
    \left\lfloor \dfrac{M}{2} \right\rfloor + 1 & \text{if } M \geq 3 \\
    \quad \quad 1 & \text{if } M \leq 2.
    \end{cases}
\end{equation}
Such behavior ensures that the initial subset contains at least half or more of the multi-modal attributes.

Notably, the formulation of the attribute uncertainty $u_{i}^{m}$ and consensus $c_{i}^{m}$ relies on correct cross-graph entity-entity correspondence.
However, the cross-graph correspondences are mismatched when E-E NC occurs, leading to unreliable $u_{i}^{m}$ and $c_{i}^{m}$.
Thus, we employ DRF only to entities $z_{i}$ that satisfy
\begin{equation}
    (1-u_{i})+c_i\geq 1.
\end{equation}
In other words, DRF is adopted to entities with reliable correspondences, which could prevent erroneous fusion.

\subsection{Evaluation Metric}
\label{appendix: metric}

We evaluate multi-modal entity alignment using Hit@1 and Mean Reciprocal Rank (MRR)~\citep{PMF}:
\begin{itemize}
\item Hit@$k$ measures the percentage of cases where the ground-truth candidate appears in the top-$k$ retrieved entity.
\item Mean Reciprocal Rank (MRR) assesses the average inverse rank of the correct entity.
\end{itemize}

\subsection{Details of Test-time Reasoning Module}
\label{appendix: TTR}
In the manuscript, we propose a test-time reasoning module to capture the underlying connections between attribute pairs. 
Here, we provide more details on the MLLM rethinking presented in Eq.~\ref{eq: mllm reasoning}.

To reduce the computational cost of MLLM reasoning, we avoid rethinking attribute pairs that satisfy the conditions as follows,
\begin{equation}
    \max (\boldsymbol{s}_{i}^{m}) \geq 0.2 \quad \lor \quad \max (\boldsymbol{s}_{i}^{m}) - \boldsymbol{s}_{ij}^{m} \geq 0.2, \ \forall s_{ij}^{m} \neq \max (\boldsymbol{s}_{i}^{m})
\end{equation}
Such behavior not only reduces the overhead of certain attribute pairs, but also encourages reasoning on unreliable pairs and thus improves MMEA performance. Notably, for the missing image attributes in the DBP15K benchmark, the corresponding rethinking scores are set to zero.

Taking the entity pair $7125-26134$ as an example, we employ the following Chain-of-Thought for the Non-name and All-attributes setting.

For a given query entity $z_{i}$, the output results of the MLLM denotes as $\boldsymbol{o}_{i}=[o_{i1};o_{i2};\cdots,o_{i10}]$, 
where each similarity satisfies $0\leq o_{ij}\leq 10$. After that, $\boldsymbol{o}_{i}$ is normalized as follow,
\begin{equation}
    \boldsymbol{\hat{o}}_{i}=(\boldsymbol{o}_{i}-5)/5.
\end{equation}
Such normalization ensures that the range of the MLLM output is consistent with that of the similarity vector $\boldsymbol{s}_{i}$ in Eq.~\ref{eq: evi}.
The $\boldsymbol{\hat{o}}_{i}$ could then employed to derive the rethinking scores via Eq.~\ref{eq: mllm reasoning}, \textit{i.e.}, $\oplus_{j \in \mathcal{T}_i^m} \left( \operatorname{CoT}\left[ x_i^{m}, (x_j')^{m}, \boldsymbol{s}_i^m \right] \right)=\boldsymbol{\hat{o}}_{i}$.

\begin{minipage}{0.99\columnwidth}
\begin{tcolorbox} 
    \raggedright
    \small
    \textbf{Chain-of-Thought for Non-name Setting}
    \begin{itemize}[leftmargin=4.5mm]
    \setlength{\itemsep}{2pt}
    
    \item \textbf{Base Prompt:}
    \begin{itemize}[leftmargin=4.5mm]
        \setlength{\itemsep}{2pt}
        \item Help me align or match entities of different knowledge graphs according to the given images and prior retrieval results.
    \end{itemize}

    \begin{center}
    \begin{minipage}{0.4\linewidth}
        \centering
        \includegraphics[width=\linewidth]{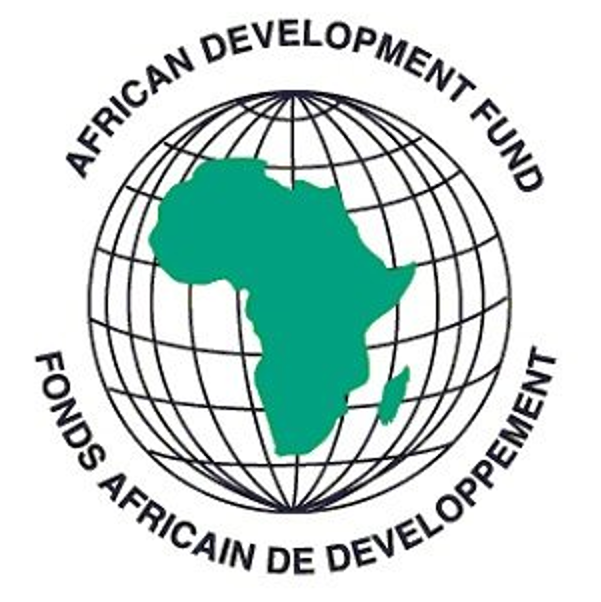}
        \small ID:7125
    \end{minipage}
    \quad
    \begin{minipage}{0.4\linewidth}
        \centering
        \includegraphics[width=\linewidth]{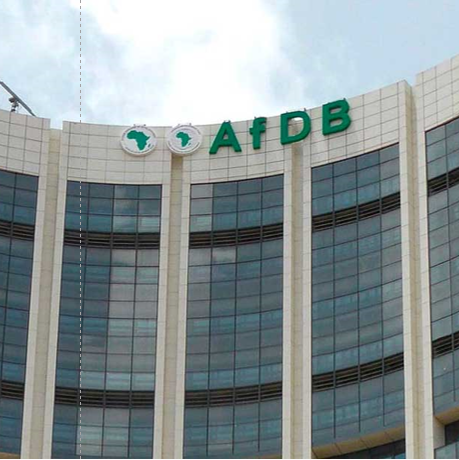}
        \small ID:26134
    \end{minipage}
    \end{center}
    

    \item \textbf{Prior Results:}
    \begin{itemize}[leftmargin=4.5mm]
        \setlength{\itemsep}{2pt}
        \item Below are prior retrieval results focusing on visual similarity of the given images.
        \item Candidate Entities List which may be aligned with QUERY Entity (ID:7125) are shown in the following list [Format: ID Similarity]:
        \begin{itemize}[leftmargin=4.5mm]
            \setlength{\itemsep}{2pt} 
            \item 22646 0.40,
            \item 22946 0.36, 
            \item 23688 0.30,
            \item 26364 0.26, 
            \item 22619 0.20, 
            \item 26052 0.14, 
            \item 26134 0.10, 
            \item 26518 0.02.
        \end{itemize}
    \end{itemize}

    \item \textbf{Rethinking Image Similarity:}
    \begin{itemize}[leftmargin=4.5mm]
        \setlength{\itemsep}{2pt}
        \item The two provided images represent the query (ID:7125) and the candidate (ID:26134). 
        \item Please evaluate the probability that the QUERY and the CANDIDATE belong to the same entity \textbf{STEP BY STEP}:   
        \item 1. Rethink the visual similarities based on the prior retrieval results and the given images.          
        \item 2. Analyze the similarities of detailed visual contents between the provided images.    
        \item 3. Consider the underlying connections between the given images.
        \item \lbrack Output Format\rbrack: \lbrack IMAGE SIMILARITY\rbrack = A out of 10, where A is in range \lbrack 0,1,2,3,4,5,6,7,8,9,10\rbrack, which represents the levels from VERY LOW to VERY HIGH. NOTICE: You MUST output strictly in this format: \lbrack IMAGE SIMILARITY\rbrack = A out of 10.
    \end{itemize}
    \end{itemize}
\end{tcolorbox}
\end{minipage}

\begin{minipage}{0.99\columnwidth}
\begin{tcolorbox} 
    \raggedright
    \small
    \textbf{Chain-of-Thought for All-attributes Setting}
    \begin{itemize}[leftmargin=4.5mm]
    \setlength{\itemsep}{2pt}
    
    \item \textbf{Base Prompt:}
    \begin{itemize}[leftmargin=4.5mm]
        \setlength{\itemsep}{2pt}
        \item Help me align or match entities of different knowledge graphs according to the given names, images and prior retrieval results.
    \end{itemize}

    \begin{center}
    \begin{minipage}{0.4\linewidth}
        \centering
        \includegraphics[width=\linewidth]{fig/case2.png}
        \small ID:7125 Name:African Development Fund
    \end{minipage}
    \quad
    \begin{minipage}{0.4\linewidth}
        \centering
        \includegraphics[width=\linewidth]{fig/case1.png}
        \small ID:26134 Name:African Development Bank
    \end{minipage}
    \end{center}
    

    \item \textbf{Prior Results:}
    \begin{itemize}[leftmargin=4.5mm]
        \setlength{\itemsep}{2pt}
        \item Below are prior retrieval results focusing on visual and textual similarity of the given images and names, respectively.  
        \item Candidate Entities List which may be aligned with QUERY Entity (ID:7125 Name:African Development Fund) are shown in the following list [Format: ID Name Similarity]:
        \begin{itemize}[leftmargin=4.5mm]
            \setlength{\itemsep}{2pt} 
            \item 26364 African Development Bank 0.42, 
            \item 22619 Southern African Development Community 0.32,
            \item 26052 Joseph Ki-Zerbo 0.24, 
            \item 22946 United Nations African Union Mission in Darfur 0.15, 
            \item 23688 OHADA 0.07, 
            \item 22646 Food and Agriculture Organization 0.05, 
            \item 26134 Jorg Asmussen -0.02, 
            \item 26518 Joyce Banda -0.08.
        \end{itemize}
    \end{itemize}

    \item \textbf{Rethinking Image Similarity:}
    \begin{itemize}[leftmargin=4.5mm]
        \setlength{\itemsep}{2pt}
        \item The two provided images represent the query (ID:7125 Name:African Development Fund) and the candidate (ID:26134 Name:African Development Bank). 
        \item Please evaluate the probability that the QUERY and the CANDIDATE belong to the same entity \textbf{STEP BY STEP}:   
        \item 1. Rethink the visual similarities based on the prior retrieval results and the given images.          
        \item 2. Analyze the similarities of detailed visual contents between the provided images.    
        \item 3. Consider the underlying connections between the given images.
        \item \lbrack Output Format\rbrack: \lbrack IMAGE SIMILARITY\rbrack = A out of 10, where A is in range \lbrack 0,1,2,3,4,5,6,7,8,9,10\rbrack, which represents the levels from VERY LOW to VERY HIGH. NOTICE: You MUST output strictly in this format: \lbrack IMAGE SIMILARITY\rbrack = A out of 10.
    \end{itemize}

    \item \textbf{Rethinking Name Similarity:}
    \begin{itemize}[leftmargin=4.5mm]
        \setlength{\itemsep}{2pt}
        \item The two provided names represent the query (ID:7125 Name:African Development Fund) and the candidate (ID:26134 Name:African Development Bank). 
        \item Based on the prior retrieval results and the given names, identify the similarities between the query entity and candidate entity.   
        \item \lbrack Output Format\rbrack: \lbrack NAME SIMILARITY\rbrack = B out of 10, where B is in range \lbrack 0,1,2,3,4,5,6,7,8,9,10\rbrack, which represents the levels from VERY LOW to VERY HIGH. NOTICE: You MUST output strictly in this format: \lbrack IMAGE SIMILARITY\rbrack = B out of 10.
    \end{itemize}
    \end{itemize}
\end{tcolorbox}
\end{minipage}

\newpage

\section{More Experiment Results}
\label{appendix: more experiments}
In this section, we provide more experimental results of the proposed RULE. Unless otherwise specified, all experiments are conducted on the ICEWS-WIKI dataset under the 50\% DNC setting.
\subsection{Results under the E-E, E-A and A-A Noisy Correspondence}
\label{appenidx: result on single NC}
In the manuscript, we have carried out experiments under DNC settings.
Here, we further provide more results under single-type NC scenarios, \textit{i.e.}, E-E, E-A, and A-A NC.
From the results in Table~\ref{tab: appendix Non-name}-\ref{tab: appendix All-attributes}, Rule significantly outperforms all baselines across various settings and datasets.

\begin{table*}[htbp]
    \centering
    \caption{Comparisons with state-of-the-art methods on Non-name benchmarks under NC setting regarding the Hits and MRR metrics.
    }
    \label{tab: appendix Non-name}
    \resizebox{1.0\linewidth}{!}{
    \begin{tabular}{c|l|ccc|ccc|ccc|ccc|ccc|c}
        \toprule
        \multirow{2}{*}{Setting} & \multirow{2}{*}{Method} & \multicolumn{3}{c|}{ICEWS-WIKI} & \multicolumn{3}{c|}{ICEWS-YAGO} & \multicolumn{3}{c|}{DBP15K$_{\text{ZH-EN}}$} & \multicolumn{3}{c|}{DBP15K$_{\text{JA-EN}}$} & \multicolumn{3}{c}{DBP15K$_{\text{FR-EN}}$} & \multirow{2}{*}{\makecell{Avg. \\ H@1}} \\
                                 &                         & H@1 & H@5 & MRR & H@1 & H@5 & MRR & H@1 & H@5 & MRR & H@1 & H@5 & MRR & H@1 & H@5 & MRR \\
        \midrule
\multirow{8}{*}{\parbox[c]{1cm}{\centering 50\% \\ E-E NC}}
& EVA         & 4.4  & 6.9  & 5.7  & 0.1  & 0.1  & 0.2  & 17.4 & 30.6 & 23.6 & 21.9 & 34.9 & 28.1 & 18.9 & 31.1 & 24.9 & 12.5 
\\
& MCLEA       & 28.6 & 43.9 & 36.2 & 20.3 & 34.7 & 27.4 & 58.2 & 75.6 & 66.1 & 58.2 & 74.6 & 65.8 & 59.1 & 75.3 & 66.6 & 44.9 \\
& XGEA        & 40.0 & 47.4 & 44.0 & 21.8 & 26.5 & 24.7 & \underline{70.1} & \underline{85.2} & \underline{75.9} & \underline{68.9} & \underline{82.1} & \underline{74.2} & \underline{68.8} & \underline{81.9} & \underline{72.0} & 53.9 \\
& MEAformer   & \underline{46.0} & \underline{62.6} & \underline{53.6} & 33.6 & 48.3 & 40.7 & 68.9 & 83.1 & 75.3 & 65.0 & 80.5 & 72.1 & 66.8 & 81.4 & 73.5 & \underline{56.1} \\
& UMAEA       & 44.0 & 60.4 & 51.9 & 28.1 & 43.5 & 35.5 & 66.8 & 82.8 & 74.0 & 62.2 & 79.4 & 70.1 & 63.7 & 80.4 & 71.2 & 53.0 \\
& PMF         & 38.6 & 53.9 & 46.0 & 32.6 & 45.4 & 38.8 & 68.4 & 83.3 & 75.2 & 67.1 & 81.8 & 73.8 & 67.5 & 82.1 & 74.2 & 54.8 \\
& HHEA & 45.6 & 59.1 & 52.1 & \underline{37.0} & \underline{48.5} & \underline{42.6} & 47.7 & 56.2 & 52.0 & 49.3 & 56.8 & 53.3 & 52.2 & 59.4 & 56.0 & 46.4 \\
& Ours        & \textbf{61.0} & \textbf{72.5} & \textbf{66.6} & \textbf{48.7} & \textbf{59.5} & \textbf{54.0} & \textbf{73.6} & \textbf{86.1} & \textbf{79.3} & \textbf{71.6} & \textbf{84.6} & \textbf{77.6} & \textbf{71.5} & \textbf{84.7} & \textbf{77.6}  & \textbf{65.3}\\
\midrule
\multirow{8}{*}{\parbox[c]{1cm}{\centering 50\% \\ E-A NC}}
& EVA         & 5.5  & 9.4  & 7.7  & 0.2  & 0.9  & 0.9  & 64.6 & 83.4 & 73.1 & 67.8 & 86.1 & 76.0 & 69.1 & 87.8 & 77.2  & 41.4\\
& MCLEA       & 34.4 & 51.9 & 42.9 & 24.0 & 39.1 & 31.5 & 75.3 & 89.6 & 81.7 & 75.4 & 90.4 & 82.0 & 76.2 & 90.9 & 82.7 & 57.1 \\
& XGEA        & 40.0 & 47.9 & 44.3 & 22.4 & 27.9 & 25.7 & \underline{76.1} & \underline{91.2} & \underline{82.6} & \underline{75.8} & 91.2 & 82.5 & 76.0 & 92.1 & 83.3 & 58.1 \\
& MEAformer   & \underline{53.1} & \underline{68.6} & \underline{60.4} & \underline{37.4} & \underline{52.8} & \underline{44.8} & 75.8 & 91.1 & 82.5 & \underline{75.8} & \underline{91.8} & \underline{82.9} & \underline{76.6} & \underline{92.6} & \underline{83.7} & \underline{63.7} \\
& UMAEA       & 50.7 & 67.5 & 58.5 & 34.3 & 49.0 & 41.6 & 72.3 & 89.8 & 80.1 & 72.2 & 90.2 & 80.2 & 73.8 & 91.6 & 81.7 & 60.7 \\
& PMF         & 44.8 & 59.4 & 51.9 & 34.9 & 48.5 & 41.5 & 74.3 & 89.7 & 81.1 & 74.4 & 90.4 & 81.5 & 76.0 & 92.3 & 83.2 & 60.9 \\
& HHEA & 49.9 & 63.7 & 56.4 & 35.8 & 48.9 & 42.1 & 46.4 & 60.8 & 53.6 & 48.1 & 61.2 & 54.6 & 52.4 & 64.0 & 58.3 & 46.0 \\
& Ours        & \textbf{62.6} & \textbf{75.5} & \textbf{68.7} & \textbf{47.7} & \textbf{59.1} & \textbf{53.5} & \textbf{80.3} & \textbf{92.4} & \textbf{85.7} & \textbf{79.1} & \textbf{92.5} & \textbf{84.9} & \textbf{80.1} & \textbf{93.2} & \textbf{85.9} & \textbf{70.0} \\
\midrule
\multirow{8}{*}{\parbox[c]{1cm}{\centering 50\% \\ A-A NC}}
& EVA         & 29.3 & 41.1 & 35.1 & 5.5  & 10.3 & 8.1  & 70.2 & 86.4 & 77.5 & 72.3 & 89.1 & 79.8 & 73.6 & 90.4 & 80.9 & 50.2 \\
& MCLEA       & 43.2 & 62.8 & 52.4 & 29.7 & 47.1 & 38.3 & 74.8 & 89.5 & 81.4 & 75.4 & 90.2 & 82.0 & 76.1 & 90.7 & 82.6 & 59.8 \\
& XGEA        & 40.8 & 48.8 & 45.1 & 24.2 & 29.7 & 27.6 & 78.2 & 92.1 & 84.4 & 78.8 & 92.9 & 84.9 & 81.0 & 93.9 & 86.6 & 60.6 \\
& MEAformer   & \underline{53.4} & \underline{69.1} & \underline{60.9} & 34.7 & 51.1 & 42.6 & \underline{81.9} & \underline{93.5} & \underline{87.0} & \underline{81.4} & \underline{94.2} & \underline{87.0} & 82.1 & 94.4 & 87.5 & \underline{66.7} \\
& UMAEA       & 49.6 & 68.0 & 58.0 & 29.5 & 45.8 & 37.6 & 77.8 & 92.1 & 84.3 & 78.3 & 92.3 & 84.5 & 79.3 & 93.5 & 85.6 & 62.9 \\
& PMF         & 51.9 & 67.3 & 59.4 & \underline{37.1} & \underline{52.5} & \underline{44.8} & 80.7 & 93.1 & 86.2 & 81.1 & 93.4 & 86.6 & \underline{82.7} & \underline{94.5} & \underline{88.0} & \underline{66.7} \\
& HHEA & 48.6 & 63.3 & 55.8 & 37.0 & 50.2 & 43.4 & 46.6 & 61.1 & 53.8 & 47.3 & 59.7 & 53.5 & 50.1 & 62.4 & 56.3 & 45.9 \\
& Ours        & \textbf{63.6} & \textbf{76.7} & \textbf{69.7} & \textbf{48.6} & \textbf{60.3} & \textbf{54.4} & \textbf{85.6} & \textbf{94.7} & \textbf{89.7} & \textbf{85.1} & \textbf{95.4} & \textbf{89.5} & \textbf{85.0} & \textbf{95.5} & \textbf{89.6} & \textbf{73.6} \\
\bottomrule
\end{tabular}
}
\end{table*}

\begin{table*}[htbp]
    \centering
    \caption{Comparisons with state-of-the-art methods on All-attributes benchmarks under NC settings regarding the Hits and MRR metrics.}
    \label{tab: appendix All-attributes}
    \resizebox{1.0\linewidth}{!}{
    \begin{tabular}{c|l|ccc|ccc|ccc|ccc|ccc|c}
        \toprule
        \multirow{2}{*}{Setting} & \multirow{2}{*}{Method} & \multicolumn{3}{c|}{ICEWS-WIKI} & \multicolumn{3}{c|}{ICEWS-YAGO} & \multicolumn{3}{c|}{DBP15K$_{\text{ZH-EN}}$} & \multicolumn{3}{c|}{DBP15K$_{\text{JA-EN}}$} & \multicolumn{3}{c}{DBP15K$_{\text{FR-EN}}$} & \multirow{2}{*}{\makecell{Avg. \\ H@1}}  \\
        & & H@1 & H@5 & MRR & H@1 & H@5 & MRR & H@1 & H@5 & MRR & H@1 & H@5 & MRR & H@1 & H@5 & MRR \\
        \midrule
        \multirow{8}{*}{\parbox[c]{1cm}{\centering 50\% \\ E-E NC}}
        & EVA            & 53.6 & 61.8 & 57.7 & 4.4  & 5.7  & 5.1  & 25.6 & 38.3 & 31.6 & 32.6 & 44.3 & 38.2 & 51.2 & 61.1 & 56.0 & 33.5 \\
        & MCLEA          & 85.0 & 92.7 & 88.5 & 86.4 & 94.5 & 90.1 & 89.2 & 94.9 & 91.8 & 94.3 & 97.7 & 95.8 & 98.1 & 99.6 & 98.8 & 90.6 \\
        & XGEA           & 57.4 & 68.8 & 62.9 & 57.0 & 68.9 & 62.7 & 87.8 & 95.2 & 91.2 & 91.6 & 97.0 & 94.1 & 96.6 & 99.0 & 97.7 & 78.1 \\
        & MEAformer      & \underline{92.7} & \underline{97.1} & \underline{94.7} & \underline{91.7} & \underline{97.0} & \underline{94.1} & \underline{94.0} & \underline{97.5} & \underline{95.6} & \underline{97.6} & \underline{99.2} & \underline{98.3} & \underline{99.2} & \underline{99.8} & \underline{99.5} & \underline{95.0} \\
        & UMAEA          & 90.1 & 96.4 & 93.0 & 87.7 & 95.4 & 91.2 & 92.6 & 97.3 & 94.7 & 96.5 & 98.9 & 97.6 & 98.5 & 99.7 & 99.1 & 93.1 \\
        & PMF            & 90.7 & 96.3 & 93.2 & 88.4 & 95.3 & 91.5 & 92.7 & 96.6 & 94.5 & 96.4 & 98.6 & 97.4 & 98.6 & 99.7 & 99.1 & 93.3 \\
        & HHEA    & 88.7 & 93.9 & 91.2 & 89.7 & 94.7 & 92.1 & 67.9 & 78.4 & 72.8 & 76.7 & 85.7 & 80.8 & 87.8 & 93.2 & 90.3 & 82.2 \\
        & Ours           & \textbf{98.4} & \textbf{99.0} & \textbf{98.7} & \textbf{97.5} & \textbf{98.6} & \textbf{98.8} & \textbf{98.0} & \textbf{96.4} & \textbf{98.2} & \textbf{98.6} & \textbf{97.3} & \textbf{99.1} & \textbf{99.6} & \textbf{99.9} & \textbf{99.8} & \textbf{98.3} \\
        \midrule
        \multirow{8}{*}{\parbox[c]{1cm}{\centering 50\% \\ E-A NC}}
        & EVA            & 64.8 & 73.0 & 68.8 & 9.1  & 10.9 & 10.1 & 69.3 & 86.0 & 76.7 & 75.6 & 89.9 & 82.0 & 88.9 & 96.1 & 92.1 & 61.6 \\
        & MCLEA          & 87.3 & 93.9 & 90.2 & 85.7 & 94.5 & 89.7 & 93.5 & 97.9 & 95.5 & 96.7 & 99.4 & 97.9 & 98.9 & 99.7 & 99.3 & 92.4 \\
        & XGEA           & 51.4 & 61.6 & 56.4 & 42.3 & 53.2 & 47.7 & 86.6 & 95.6 & 90.6 & 89.5 & 96.9 & 92.8 & 98.1 & 99.7 & 98.8 & 73.6 \\
        & MEAformer      & \underline{94.6} & \underline{98.0} & \underline{96.1} & \underline{92.0} & \underline{97.2} & \underline{94.3} & \underline{95.1} & \underline{98.6} & \underline{96.6} & 97.9 & \underline{99.7} & \underline{98.7} & 99.3 & \underline{99.8} & 99.6 & \underline{95.8} \\
        & UMAEA          & 91.1 & 96.6 & 93.6 & 86.9 & 95.1 & 90.6 & 93.1 & 98.0 & 95.3 & 96.5 & 99.2 & 97.7 & 98.6 & 99.8 & 99.2 & 93.2 \\
        & PMF            & 91.1 & 95.7 & 93.2 & 89.7 & 95.8 & 92.4 & 95.0 & 98.5 & 96.5 & \underline{98.0} & \underline{99.7} & \underline{98.7} & \underline{99.4} & \underline{99.8} & \underline{99.7} & 94.6 \\
        & HHEA    & 88.8 & 95.1 & 91.7 & 89.2 & 94.8 & 91.8 & 72.0 & 84.9 & 77.9 & 79.7 & 89.4 & 84.2 & 90.0 & 95.7 & 92.5 & 83.9 \\
        & Ours           & \textbf{98.4} & \textbf{98.9} & \textbf{98.7} & \textbf{97.5} & \textbf{98.6} & \textbf{98.7} & \textbf{98.0} & \textbf{97.7} & \textbf{99.3} & \textbf{99.4} & \textbf{98.4} & \textbf{99.4} & \textbf{99.7} & \textbf{100.0} & \textbf{99.8} & \textbf{98.6} \\
        \midrule
        \multirow{8}{*}{\parbox[c]{1cm}{\centering 50\% \\ A-A NC}}
        & EVA            & 84.6 & 91.5 & 87.9 & 61.7 & 69.8 & 65.6 & 87.0 & 95.4 & 90.7 & 91.8 & 97.8 & 94.5 & 96.4 & 99.3 & 97.6 & 84.3 \\
        & MCLEA          & 92.2 & 97.2 & 94.5 & 90.9 & 97.1 & 93.7 & 92.5 & 97.9 & 94.9 & 95.9 & 99.2 & 97.3 & 97.8 & 99.6 & 98.7 & 93.9 \\
        & XGEA           & 56.3 & 67.8 & 61.8 & 87.2 & 94.6 & 90.6 & 88.3 & 96.1 & 91.8 & 91.8 & 97.2 & 94.2 & 94.5 & 98.0 & 96.0 & 83.6 \\
        & MEAformer      & \underline{95.1} & \underline{98.3} & \underline{96.6} & \underline{93.0} & \underline{97.9} & \underline{95.2} & \underline{96.5} & \underline{99.0} & \underline{97.6} & \underline{98.6} & \underline{99.8} & \underline{99.2} & \underline{99.4} & \underline{99.8} & \underline{99.8} & \underline{96.6} \\
        & UMAEA          & 92.1 & 97.3 & 94.4 & 86.7 & 95.1 & 90.5 & 95.2 & 98.9 & 96.9 & 97.9 & 99.6 & 98.7 & 99.0 & 99.7 & 99.4 & 94.2 \\
        & PMF            & 94.6 & 98.6 & 96.4 & 91.7 & 97.1 & 94.3 & 96.2 & 99.0 & 97.5 & 98.4 & 99.7 & 99.0 & \underline{99.4} & \underline{99.8} & 99.7 & 96.2 \\
        & HHEA    & 88.7 & 95.2 & 91.8 & 89.5 & 95.5 & 92.2 & 53.4 & 68.9 & 60.7 & 64.0 & 76.6 & 70.0 & 74.1 & 83.0 & 78.3 & 73.9 \\
        & Ours           & \textbf{98.4} & \textbf{99.0} & \textbf{98.7} & \textbf{97.5} & \textbf{98.7} & \textbf{98.9} & \textbf{98.1} & \textbf{98.2} & \textbf{99.5} & \textbf{99.6} & \textbf{99.9} & \textbf{99.6} & \textbf{99.8} & \textbf{100.0} & \textbf{99.9} & \textbf{98.7} \\
        \bottomrule
    \end{tabular}
    }
\end{table*}

\newpage

\subsection{Parameter Analysis}
As mentioned in the manuscript, the hyperparameters in \textsc{Rule} include the trade-off parameter $\lambda$ in Eq.~\ref{eq: overall}, the temperature $\tau$ in Eq.~\ref{eq: evi}, and the threshold $\beta$ in Eq.~\ref{eq: threshold}. 
Here, we conduct a detailed parameter analysis to study their individual impacts.  
c in Fig.~\ref{fig: parameter analysis}, one could have the following conclusions: 
i) rule exhibits stable performance when $\lambda$ is within the range of $[2\mathrm{e}{-5}, 5\mathrm{e}{-4}]$, $\tau$ within $[0.05, 0.2]$, and $\beta$ within $[0.2, 0.4]$.  
ii) when $\lambda$ is excessively large, the model would overemphasize the regularization term in Eq.~\ref{eq: kl}, resulting in performance drop.

\begin{figure}[htbp]
    \centering
    \includegraphics[width=1.0\linewidth]{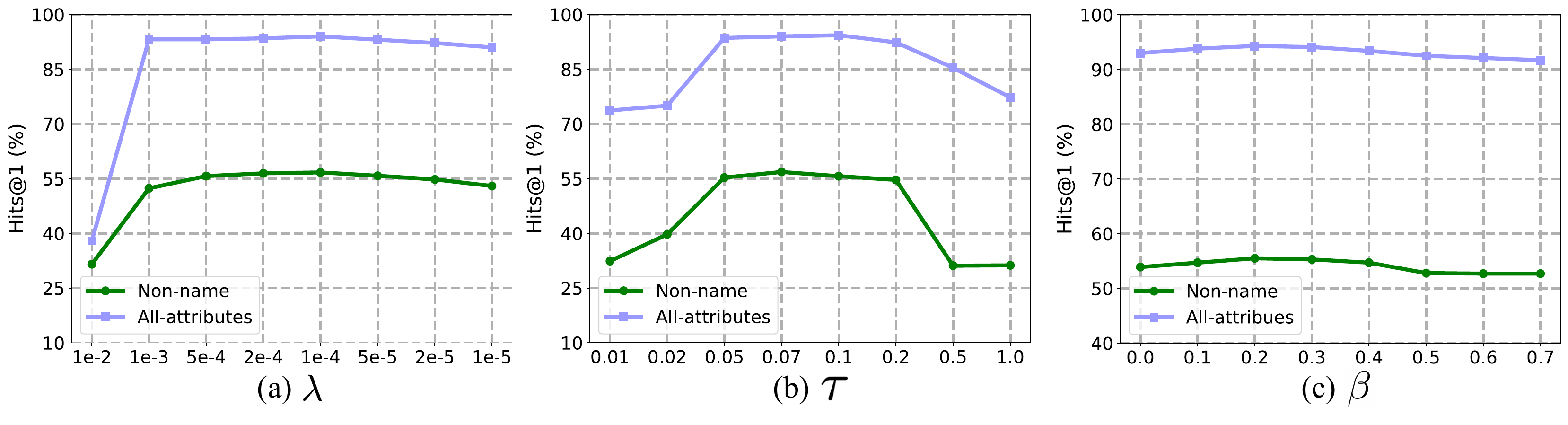}
    \caption{The parameter analysis of the trade-off parameter $\lambda$ in Eq.~\ref{eq: overall}, the temperature $\tau$ in Eq.~\ref{eq: evi}, and the threshold $\beta$ in Eq.~\ref{eq: threshold}. }
    \label{fig: parameter analysis}
\end{figure}

\subsection{More Ablation Studies}
\label{sec: More Ablation Studies}
To further verify the effectiveness of the proposed RULE, we conduct additional ablation studies of the loss terms and the test-time recitification module.
Specifically, we investigate the loss terms and the pais division module in Table~\ref{appendix: ablation study}, resulting in the following conclusions. 
First, the regularization loss penalizes the evidence of negative pairs, thereby enhancing the performance.
Second, employing tailored strategies to either $\mathcal{S}_{I}$ or $\mathcal{S}_{U}$ would contribute to the performance improvement.

\begin{table}[ht]
\centering
\caption{Ablation study of the loss terms and the correspondence division module on the ICEWS-WIKI benchmark.}
\resizebox{0.65\linewidth}{!}{
\begin{tabular}{cc|ccc|ccc}
\toprule
\multicolumn{2}{c|}{Setting} & \multicolumn{3}{c|}{Non-name} & \multicolumn{3}{c}{All-attributes} \\
$L_{DR}$ & $L_{Reg}$ & Hits@1 & Hits@5 & MRR & Hits@1 & Hits@5 & MRR \\
\midrule
\checkmark &  & 55.8 & 68.3 & 61.8 & 93.0 & 97.7 & 96.1 \\
\checkmark & \checkmark  & 56.5 & 68.6 & 62.3 & 94.0 & 97.7 & 95.7 \\
\multicolumn{2}{c|}{w/o $\mathcal{S}_{I}$} & 55.6 & 67.8 & 61.5 & 93.5 & 97.4 & 95.4 \\
\multicolumn{2}{c|}{w/o $\mathcal{S}_{U}$} & 54.8 & 67.9 & 61.1 & 93.5 & 97.2 & 95.2 \\
\bottomrule
\end{tabular}
}
\label{appendix: ablation study}
\end{table}

As discussed in the manuscript, the test-time reasoning module would help to capture the underlying connection and thus boost the performance.
Here, we carry out more ablation studies about the test-time reasoning module.
From the results in Table~\ref{tab: TTR Non-name}-\ref{tab: TTR All-attributes}, the proposed TTR module significantly improves performance on both the Non-name and All-attribute benchmarks.
It is worth noting that, despite the missing image attributes in the DBP15K benchmark, Rule still achieves stable performance gains.

\begin{table*}[htbp]
    \centering
    \caption{Ablation study of TTR module on Non-name benchmarks under different DNC settings.}
    \label{tab: TTR Non-name}
    \resizebox{1.0\linewidth}{!}{  
    \begin{tabular}{c|l|ccc|ccc|ccc|ccc|ccc}
        \toprule
        \multirow{2}{*}{Setting} & \multirow{2}{*}{Method} & \multicolumn{3}{c|}{ICEWS-WIKI} & \multicolumn{3}{c|}{ICEWS-YAGO} & \multicolumn{3}{c|}{DBP15K$_{\text{ZH-EN}}$} & \multicolumn{3}{c|}{DBP15K$_{\text{JA-EN}}$} & \multicolumn{3}{c}{DBP15K$_{\text{FR-EN}}$} \\
                                 &                         & H@1 & H@5 & MRR & H@1 & H@5 & MRR & H@1 & H@5 & MRR & H@1 & H@5 & MRR & H@1 & H@5 & MRR \\
        \midrule
        \multirow{2}{*}{\parbox[c][1cm][c]{1cm}{\centering Inherent \\ DNC}} 
        & w/o TTR & 62.1 & 76.6 & 68.8 & 46.4 & 59.6 & 53.0 & 85.5 & 94.8 & 89.7 & 85.0 & 95.3 & 89.4 & 85.2 & 95.4 & 89.7 \\
        & w TTR & 64.2 & 76.7 & 70.0 & 48.8 & 60.5 & 54.6 & 85.6 & 94.8 & 89.7 & 85.2 & 95.4 & 89.6 & 85.1 & 95.4 & 89.6 \\
        \midrule
        \multirow{2}{*}{\parbox[c]{1cm}{\centering 20\% \\ DNC}}
        & w/o TTR & 60.3 & 74.7 & 67.0 & 46.2 & 58.6 & 52.2 & 80.9 & 92.0 & 85.9 & 80.4 & 92.2 & 85.6 & 80.6 & 92.2 & 85.8 \\
        & w TTR & 62.4 & 75.1 & 68.5 & 48.3 & 59.5 & 53.9 & 81.1 & 92.0 & 86.0 & 80.5 & 92.2 & 85.6 & 80.5 & 92.2 & 85.8 \\
        \midrule
        \multirow{2}{*}{\parbox[c]{1cm}{\centering 50\% \\ DNC}}
        & w/o TTR & 56.5 & 68.6 & 62.3 & 44.5 & 56.0 & 50.3 & 73.2 & 86.0 & 79.0 & 71.5 & 84.7 & 77.6 & 71.3 & 84.7 & 77.5 \\
        & w TTR & 58.2 & 69.7 & 63.6 & 46.9 & 57.4 & 52.0 & 73.4 & 85.9 & 79.2 & 71.8 & 84.9 & 77.8 & 71.4 & 84.8 & 77.5 \\
        \bottomrule
    \end{tabular}
    }
\end{table*}

\begin{table*}[htbp]
    \centering
    \caption{Ablation study of TTR module on All-attributes benchmarks under different DNC settings.}
    \label{tab: TTR All-attributes}
    \resizebox{1.0\linewidth}{!}{  
    \begin{tabular}{c|l|ccc|ccc|ccc|ccc|ccc}
        \toprule
        \multirow{2}{*}{Setting} & \multirow{2}{*}{Method} & \multicolumn{3}{c|}{ICEWS-WIKI} & \multicolumn{3}{c|}{ICEWS-YAGO} & \multicolumn{3}{c|}{DBP15K$_{\text{ZH-EN}}$} & \multicolumn{3}{c|}{DBP15K$_{\text{JA-EN}}$} & \multicolumn{3}{c}{DBP15K$_{\text{FR-EN}}$} \\
                                 &                         & H@1 & H@5 & MRR & H@1 & H@5 & MRR & H@1 & H@5 & MRR & H@1 & H@5 & MRR & H@1 & H@5 & MRR \\
        \midrule
        \multirow{2}{*}{\parbox[c][1cm][c]{1cm}{\centering Inherent \\ DNC}} 
        & w/o TTR & 96.7 & 98.9 & 97.8 & 95.2 & 98.5 & 96.7 & 97.4 & 99.4 & 98.3 & 99.2 & 99.9 & 99.5 & 99.7 & 100.0 & 99.8 \\
        & w TTR   & 98.9 & 99.2 & 99.1 & 97.6 & 98.8 & 98.2 & 98.3 & 99.5 & 98.8 & 99.3 & 99.9 & 99.6 & 99.8 & 100.0 & 99.9 \\
        \midrule
        \multirow{2}{*}{\parbox[c]{1cm}{\centering 20\% \\ DNC}}
        & w/o TTR & 95.8 & 98.4 & 97.1 & 95.1 & 98.4 & 96.6 & 96.4 & 98.8 & 97.5 & 98.8 & 99.8 & 99.3 & 99.6 & 100.0 & 99.8 \\
        & w TTR   & 98.3 & 98.9 & 98.6 & 97.5 & 98.7 & 98.1 & 97.6 & 99.1 & 98.3 & 99.1 & 99.9 & 99.5 & 99.8 & 100.0 & 99.9 \\
        \midrule
        \multirow{2}{*}{\parbox[c]{1cm}{\centering 50\% \\ DNC}}
        & w/o TTR & 94.0 & 97.7 & 95.7 & 94.3 & 97.8 & 95.9 & 94.9 & 97.9 & 96.2 & 98.0 & 99.6 & 98.7 & 99.3 & 99.9 & 99.6 \\
        & w TTR   & 97.7 & 98.3 & 98.0 & 97.0 & 98.2 & 97.6 & 96.3 & 98.1 & 97.2 & 98.7 & 99.7 & 99.1 & 99.7 & 100.0 & 99.8 \\
        \bottomrule
    \end{tabular}
    }
\end{table*}
\subsection{More Experiments under Various DNC Ratios}
In the manuscript, we have carried out the experiments under various DNC ratios on the ICEWS-WIKI dataset.
Here, we provide more experiments under various DNC ratios on the DBP15K$_{\text{ZH-EN}}$ dataset.
As shown in Fig~\ref{fig: appenidx_ratio}, the proposed RULE significantly outperforms all the baselines.

\begin{figure}[htbp]
    \centering
    \includegraphics[width=0.6\linewidth]{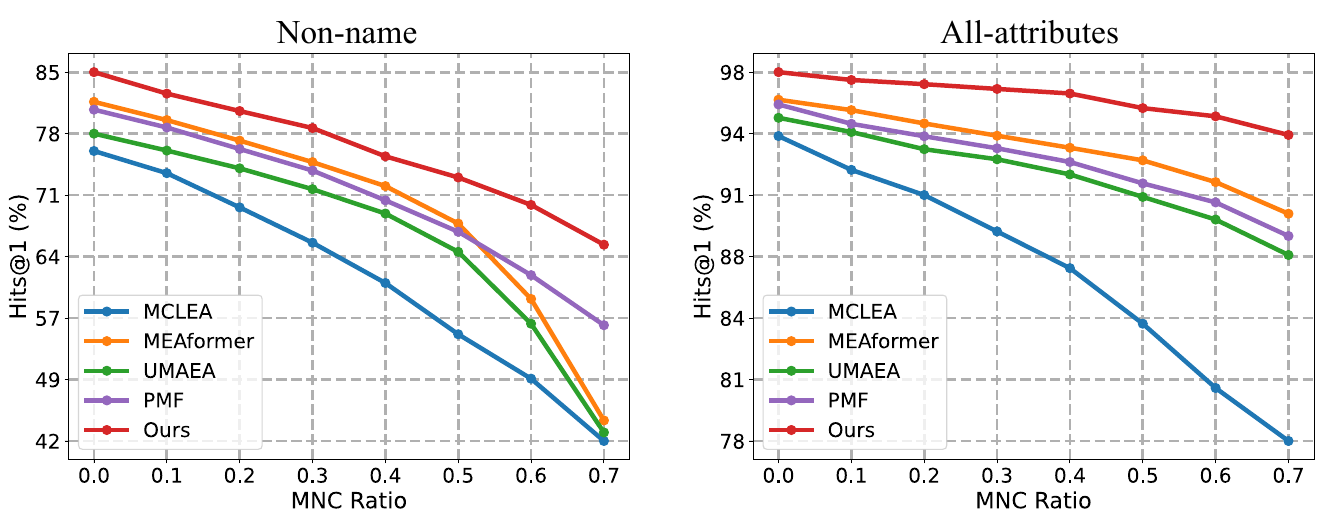}
    \caption{Performance with various DNC ratios on the  DBP15K$_{\text{ZH-EN}}$ dataset.}
    \label{fig: appenidx_ratio}
\end{figure}

\subsection{More Analytic Study of Dually Robust Fusion}
We conduct more analytic studies on the ICEWS-WIKI dataset to verify that the DRF achieves better performance by estimating the correct correspondence during the test time.
From the results in Table~\ref{tab: DRF analysis}, the estimated correspondence achieves higher accuracy compared to the vanilla fusion method, \textit{i.e.}, simple concatenation.
Consequently, Rule benefits from the estimated correspondence and the DRF module, resulting in better performance.
Note that, different from previous fusion methods~\citep{attribute_clustering,uncoupled_clustering}, the proposed DRF further explores robust fusion on the reasoning scores output by the MLLMs.

\begin{table}[htbp]
\centering
\caption{The analytic study of the estimated correspondence on the ICEWS-WIKI dataset.}
\label{tab: DRF analysis}
\resizebox{0.6\linewidth}{!}{
\begin{tabular}{l|ccc|ccc}
\toprule
\multirow{2}{*}{Setting}  & \multicolumn{3}{c|}{Non-name} & \multicolumn{3}{c}{All-attributes} \\
& Hits@1 & Hits@5 & MRR & Hits@1 & Hits@5 & MRR \\
\midrule
w/o DRF & 51.4 & 66.7 & 58.6 & 93.0 & 97.3 & 94.9 \\
Estimated & 55.3 & 65.0 & 60.2 & 94.1 & 97.4 & 95.7 \\
w DRF & 56.5 & 68.6 & 62.3 & 94.0 & 97.7 & 95.7 \\
\bottomrule
\end{tabular}
}
\end{table}

\newpage

\subsection{More Analytic Study of Inter-graph Learning}
As discussed in~\ref{sec: Ablation and Analytic Study} and \ref{sec: More Ablation Studies}, the proposed RULE improves the performance by adopting the tailored strategies to the divided subsets, \textit{i.e.}, $\mathcal{S}_{U}$, $\mathcal{S}_{I}$ and $\mathcal{S}_{C}$.
Here, we carry out more analytic studies about the statistics of the three kinds of pairs.
From the results in Fig.~\ref{fig: appedix_three_pairs} and Table~\ref{tab: statistics of three kinds}, one could observe that the proposed correspondence division module effectively distinguishes the three subsets and thus boosts the performance.
Different from previous discrepancy elimination methods~\citep{multi_modal_deviation,cross_modal_attack1,cross_modal_attack2}, the proposed DRL not only divides correspondences into positive and negative, but also further distinguishes different types of negatives and adapts the tailored strategy to mitigate the negative impacts of them.

\begin{figure}[htbp]
    \centering
    \includegraphics[width=0.7\linewidth]{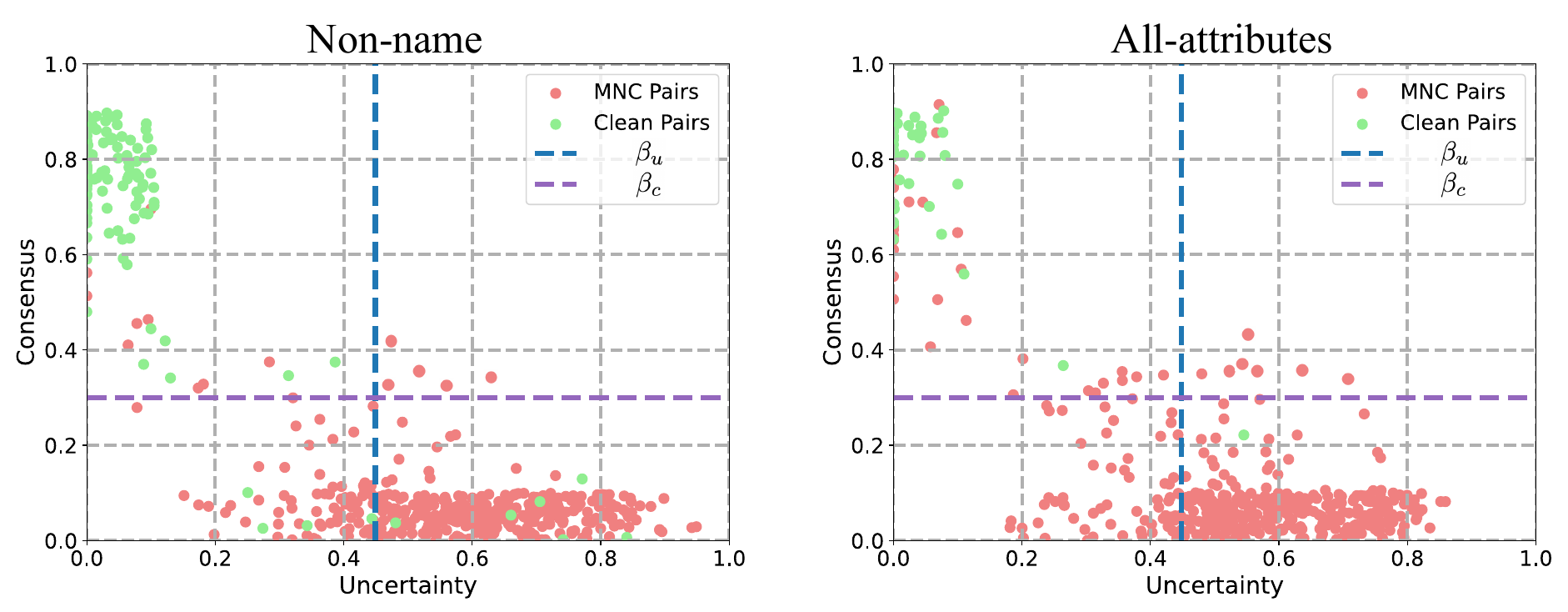}
    \caption{Quantitative analysis of the uncertainty and consensus on the integrated entity.}
    \label{fig: appedix_three_pairs}
\end{figure}

\begin{table}[htbp]
\centering
\caption{Statistics of three kinds of pairs.}
\label{tab: statistics of three kinds}
\resizebox{0.4\linewidth}{!}{
\begin{tabular}{l|ccc}
\toprule
Entity & $\mathcal{S}_{C}$ & $\mathcal{S}_{I}$ & $\mathcal{S}_{U}$ \\
\midrule
Non-name & 14.8\% & 33.9\% & 51.3\% \\
All-attributes & 14.1\% & 37.0\% & 48.9\% \\
\bottomrule
\end{tabular}
}
\end{table}

\subsection{Results of Various MLLM}
To verify the generalizability of the TTR module, we conduct additional experiments on the MLLM with various architectures and parameter scales.
Specifically, we evaluate the performance of the TTR module on Qwen2.5-VL models with 3B, 7B, and 72B parameters, as well as on LLaVA-1.6~\citep{LLAVA} with 34B parameters.
As shown in Table~\ref{tab: appendix various MLLM}, Rule achieves consistent performance improvements across various model architectures and parameter scales, demonstrating the effectiveness of the proposed TTR module.
\begin{table}[htbp]
\centering
\caption{Performance comparison of different MLLMs on Non-name and All-attributes settings.}
\label{tab: appendix various MLLM}
\resizebox{0.7\linewidth}{!}{
\begin{tabular}{l|ccc|ccc}
\toprule
\multirow{2}{*}{Methods} & \multicolumn{3}{c|}{Non-name} & \multicolumn{3}{c}{All-attributes} \\
& Hits@1 & Hits@5 & MRR & Hits@1 & Hits@5 & MRR \\
\midrule
w/o TTR & 56.5 & 68.6 & 62.3 & 94.0 & 97.7 & 95.7 \\
Qwen2.5-VL 3B & 57.2 & 68.8 & 62.7 & 96.5 & 98.2 & 97.3 \\
Qwen2.5-VL 7B & 57.4 & 69.0 & 62.9 & 97.1 & 98.2 & 97.6 \\
Qwen2.5-VL 72B & 58.2 & 69.7 & 63.6 & 97.7 & 98.3 & 98.0 \\
LLaVA-1.6 34B & 57.0 & 68.8 & 62.6 & 95.5 & 97.9 & 96.6 \\
\bottomrule
\end{tabular}
}
\end{table}
\subsection{Complexity Analysis of Test-time Correspondence Reasoning}
\label{appendix: efficiency}
To assess the efficiency of the TTR module, we report the time cost and memory cost across various parameter scales of Qwen2.5-VL. 
Note that all experiments were conducted on the NVIDIA RTX 3090 GPUs.
From the results in Table~\ref{tab: complexity_analysis}, one can observe that employing Qwen2.5-VL 3B or 7B for TTR leads to a considerable performance boost, with up to a 5× speedup. Furthermore, when resources are limited, RULE can be deployed without MLLMs. Even without TTR, RULE still significantly outperforms the strongest baselines, achieving 56.5\% vs. 43.9\% (best-performing baseline) in the ``Non-name'' setting and 94.0\% vs. 91.9\% (best-performing baseline) in the ``All-attributes'' setting. Moreover, the lightweight solutions substantially reduce the memory cost as well.
\begin{table}[htbp]
\centering
\caption{Experiment results for the complexity analysis on the ICEWS-WIKI dataset.}
\label{tab: complexity_analysis}
\resizebox{0.75\linewidth}{!}{
\begin{tabular}{l|cc|cc|c}
\toprule
\multirow{2}{*}{Methods} & \multicolumn{2}{c|}{Non-name} & \multicolumn{2}{c|}{All-attributes} & \multirow{2}{*}{Memory Consumption} \\
 & H@1 & Time Cost & H@1 & Time Cost &  \\
\midrule
w/o TTR & 56.5 & 103   & 94.0 & 109  & 1GPU $\times$ $\sim$16GB \\
Qwen2.5-VL 3B  & 57.2 & 2122  & 96.5 & 1008 & 1GPU $\times$ $\sim$8GB \\
Qwen2.5-VL 7B  & 57.4 & 2690  & 97.1 & 1329 & 1GPU $\times$ $\sim$8GB \\
Qwen2.5-VL 72B & 58.2 & 10043 & 97.7 & 4373 & 8GPU $\times$ $\sim$20GB \\
\bottomrule
\end{tabular}
}
\end{table}

\subsection{Results of Cross-entropy Implementation}
As discussed in~\citep{Evidentiallearning}, the uncertainty-based loss could be implemented using either the Mean Squared Error (MSE) loss or the Cross-Entropy (CE) loss. In Appendix~\ref{appendix: evidential learning}, we derive the uncertainty-based formulation based on MSE loss, which also serves as the foundation of the proposed dually robust loss in Eq.~\ref{eq: dually robust loss}. 
Here, we derive the dually robust loss using the cross-entropy formulation as follows,
\begin{equation}
    \begin{aligned}
    \mathcal{L}_{DR}(\boldsymbol{\alpha}_{i},\boldsymbol{\hat{y}}_{i})&= \mathbb{I}\left(u_{i}\leq\beta_{u}\right)\int \left[\sum_{j=1}^{\tilde{N}}-\hat{y}_{ij}\log (p_{ij})\right]  \ D(\boldsymbol{p}_{i} \mid \boldsymbol{\alpha}_{i}) \ d \mathbf{p}_i \\
    &=\mathbb{I}\left(u_{i}\leq\beta_{u}\right)\sum_{j=1}^{\tilde{N}} \hat{y}_{ij}\left(\psi(Q_{i})-\psi(\alpha_{ij})\right),
    \end{aligned}
    \label{eq: ce evidential learning}
\end{equation}
where $\psi(\cdot)$ is the digamma function.
To further investigate the effectiveness of the proposed DRL, we conduct additional experiments of the cross-entropy implementation based on Eq.~\ref{eq: ce evidential learning}.
From the results in Table~\ref{tab: mse_ce}, both implementations based on MSE and CE losses achieve competitive performance compared to other baselines shown in Tables~\ref{tab: main table visual} and~\ref{tab: main table textual}.
\begin{table}[htbp]
\centering
\caption{Comparisons of RULE with MAE and CE objectives under the DNC setting.}
\label{tab: mse_ce}
\resizebox{0.75\linewidth}{!}{
\begin{tabular}{ll|ccc|ccc}
\toprule
\multirow{2}{*}{Objective} & \multirow{2}{*}{Setting} & \multicolumn{3}{c|}{Non-name} & \multicolumn{3}{c}{All-attributes} \\
 & & Hits@1 & Hits@5 & MRR & Hits@1 & Hits@5 & MRR \\
\midrule
\multirow{3}{*}{MAE} 
& Inherent DNC & 62.1 & 76.6 & 68.8 & 96.7 & 98.9 & 97.8 \\
& 20\% DNC   & 60.3 & 74.7 & 67.0 & 95.8 & 98.4 & 97.1 \\
& 50\% DNC   & 56.5 & 68.6 & 62.3 & 94.0 & 97.7 & 95.7 \\
\midrule
\multirow{3}{*}{CE} 
& Inherent DNC & 61.7 & 76.2 & 68.6 & 96.3 & 98.8 & 97.5 \\
& 20\% DNC   & 59.6 & 74.0 & 66.3 & 95.9 & 98.6 & 97.2 \\
& 50\% DNC   & 54.4 & 68.5 & 61.1 & 92.5 & 97.0 & 94.5 \\
\bottomrule
\end{tabular}
}
\end{table}

\newpage

\subsection{More Analytic Study of the Reliability}
\label{appendix: ablation of reliability}
Here, we explore alternative strategies that dynamically manipulate the weights assigned to uncertainty $u_i$ and consensus $c_I$ in Eq.~\ref{eq: reliablity_weights}, \textit{i.e.},
\begin{equation}
    w_i=\left(1-u_i\right) \alpha+c_i (1-\alpha)
\end{equation}
where $\alpha$ is the balance parameter. From the results in Table~\ref{tab: weighting_strategy}, one could have the following conclusions: i) employ either uncertainty ($\alpha=1$) or consensus ($\alpha=0$) is inadequate in identifying DNC, leading to degraded performance; 
ii) thanks to the normalization in Eq.~\ref{eq: evi} and Eq.~\ref{eq: cons}, a simple linear combination of uncertainty and consensus could yield considerable performance improvements. 
The results above indicate that our design for combining uncertainty and consensus is reasonable and effective.

\begin{table}[htbp]
\centering
\caption{Analytic Experiments of the weighting strategy on the ICEWS-WIKI dataset under the ``Non-name'' setting.}
\label{tab: weighting_strategy}
\resizebox{0.3\linewidth}{!}{
\begin{tabular}{c|ccc}
\toprule
$\alpha$ & H@1 & H@5 & MRR \\
\midrule
0   & 51.8 & 62.4 & 57.0 \\
0.2 & 53.4 & 64.4 & 58.7 \\
0.4 & 54.9 & 66.5 & 60.4 \\
0.5 & \textbf{56.5} & \textbf{68.6} & \textbf{62.3} \\
0.6 & 55.1 & 67.1 & 60.8 \\
0.8 & 53.9 & 66.6 & 60.0 \\
1.0 & 52.5 & 65.7 & 58.8 \\
\bottomrule
\end{tabular}
}
\end{table}

\subsection{More Experiments on Various Backbones}
\label{appendix: various_backbones}
The proposed RULE is a model-agnostic method, which could be employed in any mainstream vision-language backbones.
To verify this, we further verify the effectiveness of RULE in various backbones.
Specifically, we adopt SigLIP~\citep{siglip} and BLIP~\citep{blip} as the backbones for extracting image and text features. 
Notably, we do not employ MLLM-based reasoning in the experiments for fairness and select the best-performing baselines from Tables~\ref{tab: main table visual}-\ref{tab: main table textual} for comparisons.
The results in Table~\ref{tab: various_backbone} demonstrate that RULE outperforms all baselines across various backbones, which confirms the generality and effectiveness of RULE.
\begin{table}[htbp]
\centering
\caption{Experiment results on the ICEWS-WIKI dataset with SigLIP and BLIP as backbone.}
\label{tab: various_backbone}
\resizebox{0.8\linewidth}{!}{
\begin{tabular}{l|l|cc|cc}
\toprule
\multirow{2}{*}{Setting} & \multirow{2}{*}{Methods} & \multicolumn{2}{c|}{SigLIP} & \multicolumn{2}{c}{BLIP} \\
 & & Non-name & All-attributes & Non-name & All-attributes \\
\midrule
\multirow{4}{*}{Inherent DNC} 
& MEAformer & 45.3 & 93.8 & 33.4 & 95.1 \\
& UMAEA     & 41.5 & 92.0 & 31.5 & 93.6 \\
& PMF       & 45.2 & 92.3 & 30.8 & 95.3 \\
& Ours      & \textbf{55.2} & \textbf{95.2} & \textbf{39.1} & \textbf{96.8} \\
\midrule
\multirow{4}{*}{20\% DNC} 
& MEAformer & 42.2 & 89.0 & 28.9 & 95.0 \\
& UMAEA     & 39.4 & 89.3 & 26.8 & 91.3 \\
& PMF       & 37.8 & 89.0 & 26.6 & 89.7 \\
& Ours      & \textbf{52.3} & \textbf{94.5} & \textbf{34.9} & \textbf{96.7} \\
\midrule
\multirow{4}{*}{50\% DNC} 
& MEAformer & 32.2 & 88.5 & 19.6 & 93.2 \\
& UMAEA     & 27.4 & 85.5 & 15.3 & 89.5 \\
& PMF       & 28.1 & 83.9 & 16.8 & 80.9 \\
& Ours      & \textbf{45.4} & \textbf{93.2} & \textbf{24.4} & \textbf{97.1} \\
\bottomrule
\end{tabular}
}
\end{table}

\newpage

\subsection{More Experiments on FB15K-DB15K and FB15K-YAGO15K Benchmarks}
Here, we conduct additional experiments on the FB15K-DB15K and FB15K-YAGO15 benchmarks with their inherent handcrafted features. 
Notably, the raw data for these two datasets is no longer available, so we don't employ MLLM for reasoning at test-time.
Here, we select the best-performing baselines from Tables~\ref{tab: main table visual}-\ref{tab: main table textual} for comparisons.
The results in Table~\ref{tab: fb15k_db15k}-\ref{tab: fb15k_yago15k} demonstrate that RULE outperforms all baselines on FB15K-DB15K and FB15K-YAGO15K datasets under the settings with various DNC ratios.

\begin{table}[htbp]
\centering
\caption{Experiment results on the FB15K-DB15K dataset under the ``Non-name'' setting. ``Inherent DNC'' refers to the setting without any additional injected noise.}
\label{tab: fb15k_db15k}
\resizebox{0.5\linewidth}{!}{
\begin{tabular}{l|l|ccc}
\toprule
Setting & Methods & H@1 & H@5 & MRR \\
\midrule
\multirow{4}{*}{Inherent DNC} 
& MEAformer & 40.8 & 62.7 & 51.0 \\
& UMAEA     & 40.8 & 62.0 & 50.7 \\
& PMF       & 43.0 & 65.3 & 53.2 \\
& Ours      & \textbf{44.4} & \textbf{63.6} & \textbf{53.8} \\
\midrule
\multirow{4}{*}{20\% DNC} 
& MEAformer & 28.1 & 49.7 & 39.6 \\
& UMAEA     & 29.5 & 51.4 & 41.1 \\
& PMF       & 29.6 & 51.4 & 41.2 \\
& Ours      & \textbf{36.0} & \textbf{56.0} & \textbf{45.9} \\
\midrule
\multirow{4}{*}{50\% DNC} 
& MEAformer & 12.0 & 26.6 & 20.7 \\
& UMAEA     & 15.0 & 32.1 & 24.7 \\
& PMF       & 15.1 & 31.7 & 24.6 \\
& Ours      & \textbf{20.3} & \textbf{36.9} & \textbf{28.9} \\
\bottomrule
\end{tabular}
}
\end{table}

\begin{table}[htbp]
\centering
\caption{Experiment results on the FB15K-YAGO15K dataset under the ``Non-name'' setting.}
\label{tab: fb15k_yago15k}
\resizebox{0.5\linewidth}{!}{
\begin{tabular}{l|l|ccc}
\toprule
Setting & Methods & H@1 & H@5 & MRR \\
\midrule
\multirow{4}{*}{Inherent DNC} 
& MEAformer & 31.8 & 50.5 & 40.7 \\
& UMAEA     & 31.2 & 50.9 & 40.3 \\
& PMF       & 34.6 & 55.0 & 44.3 \\
& Ours      & \textbf{38.9} & \textbf{55.7} & \textbf{47.0} \\
\midrule
\multirow{4}{*}{20\% DNC} 
& MEAformer & 20.5 & 38.1 & 30.1 \\
& UMAEA     & 22.6 & 40.0 & 31.8 \\
& PMF       & 23.7 & 42.0 & 33.2 \\
& Ours      & \textbf{31.7} & \textbf{49.0} & \textbf{40.4} \\
\midrule
\multirow{4}{*}{50\% DNC} 
& MEAformer & 10.3 & 20.7 & 16.2 \\
& UMAEA     & 11.9 & 24.0 & 18.6 \\
& PMF       & 11.3 & 22.8 & 17.9 \\
& Ours      & \textbf{17.9} & \textbf{31.9} & \textbf{25.1} \\
\bottomrule
\end{tabular}
}
\end{table}

\newpage

\subsection{Discussions with Active Learning Paradigm}
One intuitive approach to address DNC is to introduce additional expert annotations and the active learning paradigm~\citep{active} might be the most representative paradigm. 
However, we argue that active learning is inadequate for addressing DNC for the following reasons:
\begin{itemize}
    \item Existing MMEA-oriented active learning methods~\citep{pmcea} focus on annotating E-E correspondences, which have not explored the establishment of E-A and A-A associations. As a result, even with expert involvement, active learning is unable to solve E-A or A-A NC. It is important to emphasize that these two types of NC are highly practical in real-world scenarios, with their total proportion exceeding 40\% in commonly-used ICEWS benchmarks. 
    \item Active learning requires expert involvement, which is time-consuming and labor-intensive. In contrast, RULE provides a robust learning paradigm in an automated manner. Moreover, even after careful manual annotation, over 50\% of the data suffer from DNC challenge in real-world datasets, which indicates that expert annotation remains prone to errors.
    \item Active learning selects the most valuable samples for human annotation to improve model performance. However, most active learning methods require an initial set of labeled data for value model training before conducting active learning. As a result, the DNC challenge is still inevitable in the initial set, thus degrading the performance of the value model and then undermining the active learning.
\end{itemize}

\newpage

\section{Case Study of Dual-level Noisy Correspondence}
In this section, we showcase some DNC examples from the ICEWS benchmark in Fig.~\ref{fig: mnc_case}.
\begin{figure}[htbp]
    \centering
    \includegraphics[width=0.9\linewidth]{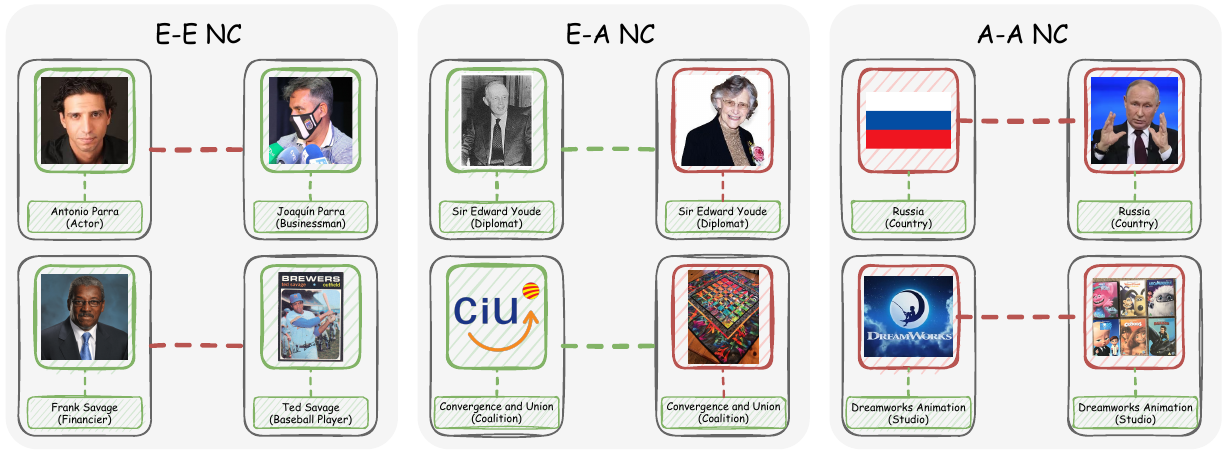}
    \caption{Examples of DNC in the ICEWS benchmark.}
    \label{fig: mnc_case}
\end{figure}

\section{Case Study of Test-time Correspondence Reasoning}
\label{appenidx: ttr_cases}
In this section, we visualize the test-time reasoning process in TTR module. 
As illustrated in Fig.~\ref{fig: ttr_name}-\ref{fig: ttr_image}, the TTR module uncovers the underlying connections between image and name attribute pairs, thus boosting performance.

\begin{figure}[htbp]
    \centering
    \includegraphics[width=0.9\linewidth]{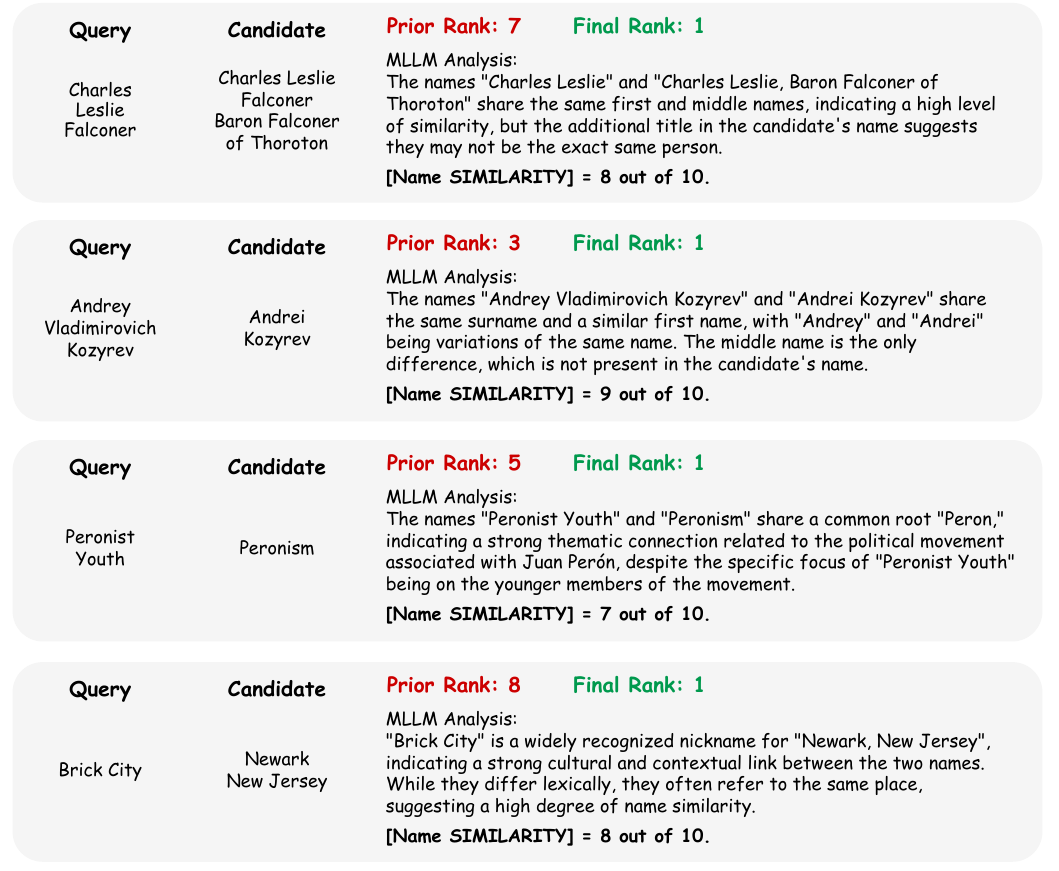}
    \caption{Reasoning process for the name attributes in test-time reasoning.}
    \label{fig: ttr_name}
\end{figure}

\begin{figure}[htbp]
    \centering
    \includegraphics[width=0.9\linewidth]{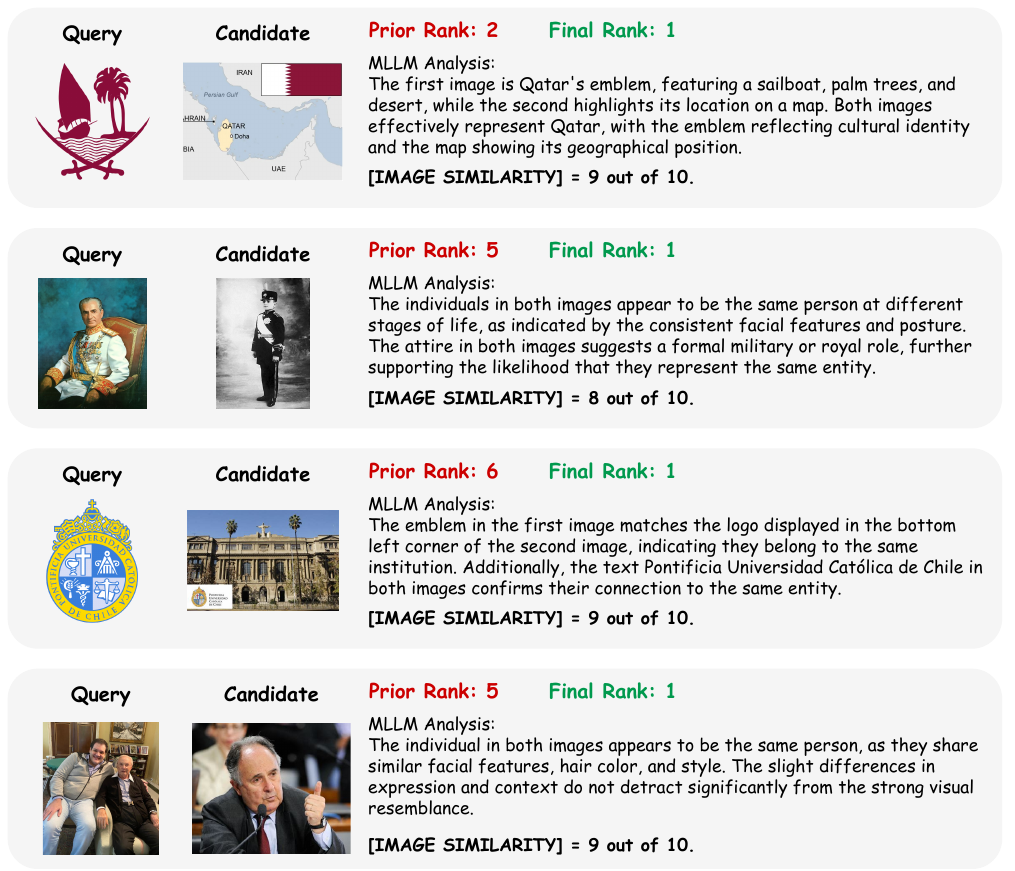}
    \caption{Reasoning process for the image attributes in test-time reasoning.}
    \label{fig: ttr_image}
\end{figure}

\end{document}